\documentclass[nohyperref]{article}

\usepackage{microtype}
\usepackage{graphicx}
\usepackage{subfigure}
\usepackage{booktabs} 
\usepackage{algorithm}
\usepackage{algorithmic}

\usepackage{hyperref}

 \usepackage[accepted]{icml2022}

\usepackage{amsmath}
\usepackage{amssymb}
\usepackage{mathtools}
\usepackage{amsthm}
\usepackage{notation}
\usepackage{bm,bbm}

\usepackage[capitalize,noabbrev]{cleveref}

\theoremstyle{plain}
\newtheorem{theorem}{Theorem}[section]
\newtheorem{proposition}[theorem]{Proposition}
\newtheorem{lemma}[theorem]{Lemma}
\newtheorem{corollary}[theorem]{Corollary}
\theoremstyle{definition}
\newtheorem{definition}[theorem]{Definition}
\newtheorem{assumption}[theorem]{Assumption}
\theoremstyle{remark}
\newtheorem{remark}[theorem]{Remark}

\usepackage[textsize=tiny]{todonotes}

\def\PP{\mathbb{P}}
\icmltitlerunning{Preference-based Reinforcement Learning}

\begin{document}

\twocolumn[
\icmltitle{Human-in-the-loop: Provably Efficient Preference-based Reinforcement Learning with General Function Approximation}



\icmlsetsymbol{equal}{*}

\begin{icmlauthorlist}
\icmlauthor{Xiaoyu Chen}{equal,af1}
\icmlauthor{Han Zhong}{equal,af2,af22}
\icmlauthor{Zhuoran Yang}{af3}
\icmlauthor{Zhaoran Wang}{af4}
\icmlauthor{Liwei Wang}{af1,af2}

\end{icmlauthorlist}

\icmlaffiliation{af1}{Key Laboratory of Machine Perception, MOE,
School of Artificial Intelligence, Peking University}
\icmlaffiliation{af2}{Center for Data Science, Peking University}
\icmlaffiliation{af22}{Peng Cheng Laboratory}
\icmlaffiliation{af3}{Department of Statistics and Data Science, Yale University}
\icmlaffiliation{af4}{Department of Industrial Engineering
and Management Sciences, Northwestern University}

\icmlcorrespondingauthor{Xiaoyu Chen}{cxy30@pku.edu.cn}
\icmlcorrespondingauthor{Han Zhong}{hanzhong@stu.pku.edu.cn}
\icmlcorrespondingauthor{Zhuoran Yang}{zhuoran.yang@yale.edu}
\icmlcorrespondingauthor{Zhaoran Wang}{zhaoranwang@gmail.com}
\icmlcorrespondingauthor{Liwei Wang}{wanglw@cis.pku.edu.cn}

\icmlkeywords{Preference-based Reinforcement Learning}

\vskip 0.3in
]



\printAffiliationsAndNotice{\icmlEqualContribution} 

\begin{abstract}
We study human-in-the-loop reinforcement learning (RL) with trajectory preferences, where instead of receiving a numeric reward at each step, the  agent only receives preferences over trajectory pairs from a human overseer. 
The goal of the agent is to learn the optimal policy which is most preferred by the human overseer. 
Despite the empirical successes, the theoretical understanding of preference-based RL (PbRL) is only limited to the tabular case.  
In this paper, we propose the first optimistic model-based algorithm for PbRL with general function approximation, which estimates the model using value-targeted regression and calculates the exploratory policies by solving an optimistic planning problem. Our algorithm achieves the regret of $\tilde{O} (\operatorname{poly}(d H) \sqrt{K} )$, where $d$ is the complexity measure of the transition and preference model depending on the Eluder dimension and log-covering numbers, $H$ is the planning horizon,  $K$ is the number of episodes, and $\tilde O(\cdot)$ omits logarithmic terms.
Our lower bound indicates that our algorithm is near-optimal when specialized to the linear setting. Furthermore, we extend the PbRL problem by formulating a novel problem called RL with $n$-wise comparisons, and provide the first sample-efficient algorithm for this new setting. To the best of our knowledge, this is the first theoretical result for PbRL with (general) function approximation.
\end{abstract}

\section{Introduction}
\label{sec: introduction}
Reinforcement learning (RL) is concerned with sequential decision-making problems in which the agent interacts with the environment to maximize its cumulative rewards. This framework has achieved tremendous successes in various fields such as Atari games \citep{mnih2013playing}, Go \citep{silver2017mastering}, and StarCraft~\citep{vinyals2019grandmaster}. While these empirical successes are encouraging, one limitation of the standard RL paradigm is that the learning algorithms and the policies crucially depend on the prior knowledge encoded in the definition of the reward function. In many real-world applications such as autonomous driving, healthcare, and robotics, reward functions might not be readily available or difficult to design, which leads to well-known challenges such as reward shaping~\citep{ng1999policy} and reward hacking~\citep{amodei2016concrete,berkenkamp2021bayesian}.

If we have enough demonstrations of the desired task, one possible solution to address the above problems is to extract a reward function using inverse reinforcement learning~\citep{ng2000algorithms,abbeel2004apprenticeship}. This reward function can be further used to train an agent with reinforcement learning algorithms. More directly, we can use imitation learning~\citep{ho2016generative,hussein2017imitation,osa2018algorithmic} to clone the demonstrated behavior. However, these approaches are not applicable to situations where the demonstration data from experts are expensive to obtain, or behaviors are difficult for humans to demonstrate.

Another popular alternative to handle the lack of reward functions is called Preference-based Reinforcement Learning (PbRL)~\citep{busa2014preference,wirth2017survey}. In PbRL, instead of observing the reward information on the encountered state-action pairs, the agent only receives 1 bit preference feedback over a trajectory pair from an expert or a human overseer. Such preference feedback is often more natural and straightforward to specify in many RL applications, especially those involving human evaluations. This learning paradigm has been widely applied to multiple areas, including robot training~\citep{jain2013learning,jain2015learning,christiano2017deep}, game playing~\citep{wirth2012first,wirth2014learning} and clinical trials~\citep{zhao2011reinforcement}.

Despite its promising application in various areas, the theoretical understanding of PbRL is limited to only the tabular RL setting. \citet{novoseller2020dueling} proposes the Double Posterior Sampling method using Bayesian linear regression with an asymptotic regret sublinear in $T$ (the number of time steps). \citet{xu2020preference} presents the first finite-time analysis for PbRL problems with near-optimal sample complexity bounds. \citet{pacchiano2021dueling} studies the regret minimization problem for PbRL with linearly-parameterized preference function. However, all the previous algorithms are restricted to the tabular setting, and their complexity bounds scale polynomial dependence on the cardinality of the state-action space. Therefore, their algorithms can fail in the more practical scenario where the state space is extremely large.


Recently, there have been tremendous results studying standard RL problem with general function approximation (e.g. ~\citet{du2019provably,yang2019sample,ayoub2020model,wang2020reinforcement,jin2020provably,zanette2020learning,wang2020reward,zhou2021provably, chen2021near, zhou2021nearly}). However, their algorithms cannot be directly applied to PbRL setting due to the following two reasons. Firstly, most of their algorithms estimate the value function by utilizing the Bellman update with general function approximation. However, we cannot estimate the value of certain policies since the reward values are hidden and unidentifiable up to shifts in rewards. Secondly, since the preference feedback in PbRL depends on the utility of the whole trajectories and can be even non-Markovian, the optimal policy for PbRL problems can be possibly history-dependent. This violates the fundamental requirement of the Markovian policy class in the standard RL setting.


In this work, we tackle the regret minimization problem for preference-based reinforcement learning with general function approximation. Specifically, we study the PbRL problem where both the unknown transition model and the unknown preference function are known to belong to given function spaces. The function spaces are general sets of functions, which may be either finitely parameterized or nonparametric. This setting is more general than the previous theoretical results for PbRL~\citep{novoseller2020dueling,xu2020preference,pacchiano2021dueling}. Our contributions are summarized as follows:
\begin{itemize}
    \item We propose a statistically efficient algorithm called Preference-based Optimistic Planning (PbOP) for PbRL with general function approximation. We prove that the regret of our algorithm is $\tilde{O}\left(\operatorname{poly}(d H) \sqrt{K}\right)$, where $d$ is the complexity measure of the transition and preference model depending on the Eluder dimension~\citep{russo2013eluder} and log-covering numbers, H is the planning horizon, and $K$ is the number of episodes. Additionally, we find that our algorithm can be almost directly applied to a setting called RL with once-per-episode feedback~\citep{efroni2020reinforcement,chatterji2021theory}, which is another reinforcement learning problem dealing with the lack of reward functions.
    \item We present a reduction from the setting of RL with once-per-episode feedback to PbRL. When specialized to the linear case, we prove a nearly-matching lower bound for PbRL based on this reduction. The lower bound indicates that our algorithm is near-optimal in the case of linear function approximation.
    \item We formulate a novel setting called RL with $n$-wise comparisons to cover situations where multiple trajectories are sampled and compared with each other in each episode. This setting is more general than the standard PbRL setting and covers many real situations including robotics and clinical trails. Based on our PbOP algorithm, we also propose an algorithm with near-optimal regret.
\end{itemize}
\section{Related Work}

\textbf{Preference-based RL}\ We refer readers to \citet{wirth2017survey} for an overview of Preference-based RL. Overall, there are three different types of preference feedback in the PbRL literature. Firstly, the preferences can be defined on the action space where the labeler tells which action is better for a given state. Secondly, state preference  determines the preferred state between state pairs, which indicates that there is an action in the preferred state that is better than all actions available in the other state. Lastly, trajectory preference compares between trajectory pairs and specifies that a trajectory should be preferred over the other ones. Trajectory preference is the most general form of preference-based feedback, which is also the main focus of this work. As discussed in the introduction, previous theoretical results studying PbRL with trajectory feedback mainly focus on the tabular RL setting with finite state and action space~\citep{novoseller2020dueling,xu2020preference,pacchiano2021dueling}. 
Besides PbRL, preference-based learning has also been well-explored in bandit setting 
under the notion of ``dueling bandits''~\citep{yue2012k,falahatgar2017maxing,falahatgar2017maximum,busa2018preference,xu2020zeroth,busa2018preference}, which can be regarded as a special case of PbRL with single state and horizon $H=1$.

\textbf{RL with Function Approximation}\ Recently, there are a large number of theoretical results about provably efficient exploration in the standard RL problem with function approximation, The most basic and frequently explored setting is RL with linear function approximation.  See e.g., \citet{du2019provably,yang2019sample,jin2020provably,cai2020provably,zanette2020learning,wang2020reward, zhou2021nearly,zhou2021provably} and references therein. Beyond linear setting, there are also results studying RL with general function approximation \citep{jiang2017contextual, wang2020reinforcement,kong2021online,foster2020instance,jin2021bellman,du2021bilinear}. Our work is mostly relevant to previous works studying provably efficient model-based RL with general function approximation. For example, \citet{osband2014model} makes explicit model-based assumption that the transition operator and the reward function lie in a given function class, and analyse the regret of Thompson sampling when applied to RL with general function approximation. \citet{ayoub2020model} proposes an algorithm for episodic model-based RL based on value-targeted regression. Recently, \citet{chen2021understanding} extends the model-based algorithm for episodic RL to the infinite-horizon setting.

\textbf{RL with Once-per-episode Feedback} \ RL with once-per-episode feedback is another reinforcement learning paradigm to deal with the lack of a reward function in various real-world scenarios, in which the agent only receives non-Markovian feedback based on the whole trajectory at the end of an episode. \citet{efroni2020reinforcement} firstly studies the setting with the assumption of inherent Markovian rewards. They propose a hybrid optimistic-Thompson Sampling approach with a $\sqrt{K}$ regret. \citet{chatterji2021theory} removes the Markovian rewards assumption and provide optimistic algorithms based on the well-known UCBVI algorithm~\citep{azar2017minimax}. Though aimed to tackle the similar problem, the setting of RL with once-per-episode feedback is relatively independently studied compared with the PbRL, and this is the first work that points out the connections between two different settings.

\section{Preliminaries} \label{sec: preliminaries}


Throughout this paper, we use the following notations. For any positive integer $n$, we use $[n]$ to denote $\{1, 2, \cdots, n\}$. We denote by $\Delta(\mathcal{A})$ the set of probability distributions on a set $\caA$.

\subsection{PbRL with Trajectory Preferences}
We study the episodic finite horizon Markov decision process (MDP), which is defined by a tuple $(\mathcal{S}, \mathcal{A}, H, \mathbb{P})$, where $\mathcal{S}$ and $\mathcal{A}$ are state and action spaces, respectively, $H$ is the horizon of the MDP, and $\mathbb{P}: \mathcal{S} \times \mathcal{A} \rightarrow \Delta(\mathcal{S})$ is the transition kernel. In the episodic setting, the agent interacts with the environment for $K$ episodes. Each episode consists of $H$ steps. For ease of presentation, we assume the initial state of each episode is a fixed state $s_1 \in \mathcal{S}$. We remark that this setting can be generalized to the setting that the initial state is sampled from a fixed distribution.  At the beginning of the $k$-th episode, the agent determines two policies $(\pi_{k, 1}, \pi_{k,2})$. After executing these two policies, the agent obtains two trajectories $\{\tau_{k, i} = (s_{k,1, i}, a_{k, 1, i}, s_{k, 2, i}, a_{k, 2, i}, \cdots, s_{k, H, i}, a_{k, H, i})\}_{i = 1, 2}$. In PbRL, unlike the standard RL where the agent can receive reward signals, the agent can only obtain the preference $o_k$ between two trajectories $(\tau_{k,1}, \tau_{k,2})$. Here $o_k$ is a Bernoulli random variable with $\Pr(o_k = 1) = \Pr(\tau_{k, 1} >\tau_{k, 2})$. For ease of presentation, we denote $\mathbb{T}(\tau_1,\tau_2) = \Pr(\tau_1 >\tau_2)$ for any two trajectories $\tau_1$ and $\tau_2$.

We define $\Pi$ to be the set containing all  history-dependent policies, and we use $\operatorname{Traj}$ to denote the set of $H$-step trajectories. 
For any transition $\mathbb{P}$ and policy $\pi$, we also use $\tau \sim (\mathbb{P},\pi)$ to denote that the trajectory $\tau$ is sampled using policy $\pi$ from the MDP with transition $\mathbb{P}$. 
With a slight abuse of notation, we define $\mathbb{T}(\pi_1,\pi_2) = \mathbb{E}_{\tau_{1} \sim (\mathbb{P}, \pi_{1}), \tau_2 \sim ( \mathbb{P}, \pi_2)}\mathbb{T}(\tau_1,\tau_{2})$. Throughout this paper, we make the following assumption.
\begin{assumption} \label{assumption: definition of pistar}
There exists a policy $\pi^*$, such that $\mathbb{T}(\pi^*,\pi_{0}) \geq \frac{1}{2}, \forall \pi_0 \in \Pi$.
\end{assumption}

This is an extension of the optimal-arm assumption in dueling bandits~\citep{busa2018preference}. It is also more general than Assumption 1 in \cite{xu2020preference}.

We study the regret minimization problem, where the regret is defined as: 
$Reg(K) = \sum_{k=1}^{K} \sum_{i=1}^2 \left(\mathbb{T}(\pi^*,\pi_{k,i}) -\frac{1}{2}\right)$.
By Assumption~\ref{assumption: definition of pistar}, we have the regret is non-negative, and our goal is to design algorithms with sub-linear regret guarantees.

\subsection{General Function Approximation}
In this subsection, we introduce  notions that can characterize the complexity of the function class. 

\textbf{Covering Number} \ When the function class has infinite elements, we usually use the covering number to capture the complexity.

\begin{definition}[Covering Number]
      The $\epsilon$-covering number of a set $\mathcal{F}$ under metric $d$, denoted as $\mathcal{N}(\mathcal{F}, \epsilon, d)$, is the minimum integer $m$ such that there exists a subset $\mathcal{F}' \subset \mathcal{F}$ with $|\mathcal{F}| = m$, and for any $f \in \mathcal{F}$, there exists some $f' \in \mathcal{F}'$ satisfying $d(f, f') \le \epsilon$.
\end{definition}

\textbf{Eluder Dimension} \
We use the concept of Eluder dimension introduced by \citet{russo2013eluder} to characterize the complexity of different function classes in RL.
\begin{definition}[$\alpha$-independent]
      Let $\mathcal{F}$ be a function class defined in $\mathcal{X}$, and $\{x_1, x_2, \cdots, x_n\} \in \mathcal{X}$. We say $x \in \mathcal{X}$ is $\alpha$-independent of $\{x_1, x_2, \cdots, x_n\}$ with respect to $\mathcal{F}$ if there exists $f_1, f_2 \in \mathcal{F}$ such that $\sqrt{\sum_{i = 1}^n(f_1(x_i) - f_2(x_i))^2} \le \alpha$, but $f_1(x) - f_2(x) \ge \alpha$.
\end{definition}

\begin{definition}[Eluder Dimension]
      Suppose $\mathcal{F}$ is a function class defined in $\mathcal{X}$, the $\alpha$-Eluder dimension is the longest sequence $\{x_1, x_2, \cdots, x_n\} \in \mathcal{X}$ such that there exists $\alpha' \ge \alpha$ where $x_i$ is $\alpha'$-independent of $\{x_1, \cdots, x_{i-1}\}$ for all $i \in [n]$.
\end{definition}

\textbf{Preference Function} \
We assume the function $\mathbb{T}(\tau_1,\tau_2)$ belongs to the function space $\caF_{\mathbb{T}}$, which is defined as
\begin{align}
    \caF_{\mathbb{T}} = \left\{f(\tau_1,\tau_2) \in [0,1], f(\tau_1, \tau_2) + f(\tau_2, \tau_1) = 1\right\}.
\end{align}
We assume the $\alpha$-Eluder dimension of the function class $\caF_{\mathbb{T}}$ is bounded by $d_{\mathbb{T}}$. Here $\alpha$ is a parameter which will be specified later.

Following the analysis by \citet{russo2013eluder}, we can show that the function apace $\caF_{\mathbb{T}}$ has bounded Eluder dimension in linear and generalized linear cases.
\begin{remark}[Linear Preference Models]
\label{remark: linear preference}
Consider the case of $d$-dimensional linear preference models $f(\tau_1,\tau_2) = \psi(\tau_1,\tau_2)^{\top} \theta$ where $\psi: \operatorname{Traj} \times \operatorname{Traj} \rightarrow \mathbb{R}^d$ is a known feature map satisfying $\|\psi(\tau_1,\tau_2)\|_2 \leq L$ and $\theta \in \mathbb{R}^d$ is an unknown parameter with $\|\theta\|_2 \leq S$. Then the $\alpha$-Eluder dimension of $\caF_{\mathbb{T}}$ is at most $O(d\log(LS/\alpha))$.
\end{remark}

\begin{remark}[Generalized Linear Preference Models]
\label{remark: generalized linear preference}
For the case of $d$-dimensional generalized linear models $f(\tau_1,\tau_2) = g(\langle\phi(\tau_1,\tau_2), \theta\rangle)$ where $g$ is an increasing Lipschitz continuous function, $\psi: \operatorname{Traj} \times \operatorname{Traj} \rightarrow \mathbb{R}^d$ is a known feature map satisfying $\|\psi(\tau_1,\tau_2)\|_2 \leq L$ and $\theta \in \mathbb{R}^d$ is an unknown parameter with $\|\theta\|_2 \leq S$. Set $\bar{h}=\sup _{\tilde{\theta}, \tau_1,\tau_2} g^{\prime}(\langle\phi(\tau_1,\tau_2), \tilde{\theta}\rangle)$, $\underline{h}=\inf _{\tilde{\theta}, \tau_1,\tau_2} g^{\prime}(\langle\phi(\tau_1,\tau_2), \tilde{\theta}\rangle)$ and  $r = \bar{h}/\underline{h}$. Then the $\alpha$-Eluder dimension of $\caF_{\mathbb{T}}$ is at most $O(dr^2\log(rLS\bar{h}/\alpha))$. Therefore, our results subsume the setting of logistic preference functions~\citep{pacchiano2021dueling} as a spacial case.
\end{remark}

\textbf{Model Complexity} \
Similar to \citet{ayoub2020model}, we use Eluder dimension to characterize the complexity of the model class. We denote $\mathcal{V} = \{f: \mathcal{S} \rightarrow [0, 1]\}$.
We assume the real transition $\mathbb{P}$ belongs to a transition set $\mathcal{P}$, and we define the function space $\mathcal{F}$ as the collections of functions $f : \caS \times \caA \times \caV \rightarrow \mathbb{R}$:
\begin{equation}
\begin{aligned}
    \caF_{\mathbb{P}} = &\Big\{f \mid  \exists \mathbb{P} \in \mathcal{P}, s.t. \, \forall (s, a, v) \in \caS \times \caA \times \caV,  \\
    & \qquad f(s, a, v)=\int \mathbb{P}\left(d s^{\prime} \mid s, a\right) v\left(s^{\prime}\right)  \Big\}. 
\end{aligned}
\end{equation}
We define $d_{\mathbb{P}} = dim_{\caE}(\caF_{\mathbb{P}},\alpha)$ to be the $\alpha$-Eluder dimension of $\caF_{\mathbb{P}}$.

\begin{remark}[Linear Mixture Models]
\label{remark: linear mixture model}
Such a model subsumes linear mixture models as a special case~\citep{ayoub2020model}. Specifically, we say an MDP is a linear mixture MDP if 
$$\mathcal{P} = \{ \psi(s, a, s')^\top \theta : \theta \in \Theta\},$$
where $\psi : \caS \times \caA \times \caS \rightarrow \mathbb{R}^d$ is a known feature map satisfying $\|\sum_{s'}\psi(s,a,s')V(s')\|_{2} \leq 1, \forall s \in \caS, a \in \caA, V \in \caV$, and $\theta \in \Theta$ satisfies $\|\theta\|_2 \leq B$ for a constant $B$. For linear mixture models, the $\alpha$-Eluder dimension of $\caF_{\mathbb{P}}$ is of order $O(d\log(B/\alpha))$.
\end{remark}



\subsection{RL with Once-per-episode Feedback} \label{sec:prelim:once:per:episode}
In this subsection, we introduce another related RL setting called RL with once-per-episode feedback.

In the $k$-th episode, the agent executes a policy $\pi_k$ and obtains a trajectory $\tau_k$ induced by $\pi_k$. At the end of the episode, the agent receives a feedback $y_k \in \{0, 1\}$, where $y_k = g^*(\tau_k)$ for some unknown function $g^*$. We assume $g^*$ belongs to the function space $\mathcal{F}_{\mathbb{G}}$, which is defined as $\mathcal{F}_{\mathbb{G}} = \{g: \operatorname{Traj} \rightarrow [0,1]\}. $
We assume the $\alpha$-Eluder dimension of the function class $\caF_{\mathbb{G}}$ is bounded by $d_{\mathbb{G}}$. Here $\alpha$ is a parameter which will be specified later.
With slight abuse of notations, we denote $g^*(\pi) = \mathbb{E}_{\tau \sim (\mathbb{P}, \pi)}[g^*(\tau)]$. For any policy $\pi$, we can define the preference value function as $V_1^{\pi}(s_1) = \mathbb{E}_{\tau \sim (\mathbb{P}, \pi)}[g^*(\tau)].$ Our goal is to minimize the regret $Reg(K) = \sum_{k = 1}^K V_1^{*}(s_1) - V_1^{\pi_k}(s_1)$,
where $V_1^{*}(s_1) = \max_{\pi} V_1^{\pi}(s_1)$ is the value of the optimal policy.
\begin{remark}
Similar to Remarks \ref{remark: linear preference} and \ref{remark: generalized linear preference}, our setting incorporates (generalized) linear case. Hence, our following results can be naturally applied to the scenario considered by~\citet{chatterji2021theory}.
\end{remark}

\section{Main Results for Preference-based RL}
In this section, we present the main results for preference-based RL. We first propose a novel algorithm called Preference-based Optimistic Planning (PbOP) and establish the regret upper bound for it. To show the sharpness of our result, we also prove an information-theoretic lower bound in the linear case.

\subsection{Algorithm}
The algorithm is formally defined in Algorithm~\ref{alg: PbRL}. Overall, we employ the standard least-squares regression to learn the transition dynamics and the preference function. In each episode, we first update the model estimation based on the history samples till episode $k-1$. We define the confidence sets and calculate the confidence bonuses for the transition and preference estimations, respectively. Based on the confidence sets and the bonus terms, we maintain a policy set in which all policies are near-optimal with minor sub-optimality gap with high probability. Finally, we execute the most exploratory policy pair in the policy set and observe the preference between the trajectories sampled using these two policies.

\begin{algorithm}
  \caption{PbOP: Preference-based Optimistic Planning}
  \label{alg: PbRL}
    \begin{algorithmic}[1]
        \STATE Set $\beta_{\mathbb{T}} = \beta_{\mathbb{P}} = 8 \log(2K \mathcal{N}\left(\mathcal{F}_{\mathbb{T}}, 1/K,\|\cdot\|_{\infty}\right)/ \delta)$
        \FOR{episode $k = 1,\cdots, K$}
              \STATE Calculate the estimation $\hat{\mathbb{T}}_{k}$ and $\hat{\mathbb{P}}_k$ using least-squares regression (Eqn.~\eqref{eqn: definition of hatT} and ~\eqref{eqn: definition of hatP}) 
              \STATE Construct the high-confidence set $\caB_{\mathbb{T},k}$ for the preference $\mathbb{T}$ (Eqn.~\eqref{eqn: confidence set for T})
              \STATE Construct the high-confidence set $\caB_{\mathbb{P},k}$ for transition $\mathbb{P}$ (Eqn.~\eqref{eqn: confidence set for P})
              \STATE Define the preference bonus $b_{\mathbb{T},k}(\tau_1,\tau_2) = \max_{f_1, f_2 \in \caB_{\mathbb{T},k}}\left|f_1(\tau_1,\tau_2) - f_2(\tau_1,\tau_2)\right|$
              \STATE Define the transition bonus $b_{\mathbb{P},k}(s, a, V) = \max_{\mathbb{P}_1,\mathbb{P}_2 \in \caB_{\mathbb{P},k}} (\mathbb{P}_1-\mathbb{P}_2)V(s,a)$
              \STATE Set $b_{\mathbb{P},k}(s, a) = \max_{V \in \mathcal{V}} b_{\mathbb{P},k}(s, a, V)$
              \STATE Construct the policy set $\caS_k$ as Eqn.~\eqref{eqn: definition of S_k}
              \STATE Compute policies $(\pi_{k,1},\pi_{k,2})$ as Eqn.~\eqref{eqn: compute policies}
              \STATE Execute the policy $\pi_{k,1}$ and $\pi_{k,2}$ for one episode, respectively, and then observe the trajectory $\tau_{k,1}$ and $\tau_{k,2}$
              \STATE Receive the preference $o_k$ between $\tau_{k,1}$ and $\tau_{k,2}$
      \ENDFOR
    \end{algorithmic}
  \end{algorithm}

\textbf{Confidence Sets and Bonuses} \ We first explain our construction of the confidence sets and bonus terms in Algorithm~\ref{alg: PbRL}. For the preference function, the estimation $\hat{\mathbb{T}}_k$ is defined as the minimizer of the following least-squares loss:
\begin{align}
\label{eqn: definition of hatT}
    \hat{\mathbb{T}}_k = \argmin_{\mathbb{T}^{\prime} \in \mathcal{F}_{\mathbb{T}}} \sum_{t=1}^{k-1} \left({\mathbb{T}}^{\prime}({\tau}_{t,1},{\tau}_{t,2})-o_t\right)^2.
\end{align}
We can guarantee that the real preference function $\mathbb{T}$ is contained in the following high-probability set for $\mathbb{T}$:
\begin{align}
\label{eqn: confidence set for T}
    \caB_{\mathbb{T},k} = \left\{ \mathbb{T}' \mid  \sum_{t=1}^{k-1}\left(\hat{\mathbb{T}}_{k} - \mathbb{T}'\right)^2({\tau}_{t,1}, {\tau}_{t,2}) \leq \beta_{\mathbb{T}}\right\}.
\end{align}
We define the exploration bonus $b_{\mathbb{T}, k}\left(\tau_{1}, \tau_{2}\right)$ to be the ``width'' of the confidence set $\caB_{\mathbb{T},k}$, i.e., $b_{\mathbb{T},k}(\tau_1,\tau_2) = \max_{f_1, f_2 \in \caB_{\mathbb{T},k}}\left|f_1(\tau_1,\tau_2) - f_2(\tau_1,\tau_2)\right|$. This bonus measures the uncertainty of a certain trajectory pair $(\tau_1,\tau_2)$ w.r.t. the confidence set $\caB_{\mathbb{T},k}$, which will be used to calculate the near-optimal policy set and the most exploratory policy pairs later.

When it comes to the confidence set and the bonus term for the transition estimation, the situation becomes slightly complicated. Recall that in the definition of the transition function $f_{\mathbb{P}} \in \caF_{\mathbb{P}}$, the input variables includes both the state-action pair $(s,a)$ and the next-step target function $V$. Given history samples $\{\tau_{t, i} = (s_{t,1, i}, a_{t, 1, i}, s_{t, 2, i}, a_{t, 2, i}, \cdots, s_{t, H, i}, a_{t, H, i})\}_{i = 1, 2,t \in [k-1]}$, we can possibly estimate the transition model by minimizing the least-squares loss with respect to certain value target $\{V_{t,h,i}\}_{i=1,2,h \in [H], t \in [k-1]}$:
\begin{align}
    \label{eqn: definition of hatP}
    \hat{\mathbb{P}}_{k} =\operatorname{argmin}_{\mathbb{P}^{\prime} \in \mathcal{P}} \sum_{i=1}^{2} \sum_{t=1}^{k-1} \sum_{h=1}^{H}& \big(\langle \mathbb{P}^{\prime}(\cdot \mid s_{t,h,i}, a_{t,h,i}), V_{k,h,i}\rangle \notag\\
    &-V_{k,h,i}(s_{k,h+1,i})\big)^{2}.
\end{align}

Now the remaining problem is how to define the target function $V_{t,h,i}$. In the problem of standard RL with general function approximation, \citet{ayoub2020model} uses the optimistic value estimation as the target function in each episode. However, in the PbRL setting, since the reward information is hidden, we cannot calculate the value estimation for each given state-action pairs. To tackle this problem, we define the bonus $b_{\mathbb{P},k}(s,a,V)$ for any $s \in \caS, a \in \caA$ and the target function $V \in \mathcal{V}$, and use $V = \argmax_{V \in \mathcal{V}} b_{\mathbb{P},k}(s,a,V)$ as the regression target for the state-action pair $(s,a)$. Similar ideas have also been applied to the problem of reward-free exploration for linear mixture MDPs~\citep{zhang2021reward,chen2021near}. 

To be more specific, we define $L_{k}(\mathbb{P}_1,\mathbb{P}_2)$ as
\begin{align}
    L_{k}\left(\mathbb{P}_1, {\mathbb{P}}_{2}\right)=\sum_{i=1}^2 &\sum_{t=1}^{k-1} \sum_{h=1}^{H}\big(\langle \mathbb{P}_1\left(\cdot \mid s_{t,h,i}, a_{t,h,i}\right) \notag\\
    & -{\mathbb{P}}_{2}\left(\cdot \mid s_{t,h,i}, a_{t,h,i}\right), V_{t,h,i}\rangle\big)^{2}.
\end{align}
We construct the high confidence set for transition $\mathbb{P}$:
\begin{align}
\label{eqn: confidence set for P}
    \caB_{\mathbb{P},k} = \left\{ \mathbb{P}' \mid L_k(\mathbb{P}', \hat{\mathbb{P}}_k) \leq \beta_{\mathbb{P}}\right\},
\end{align}
The exploration bonus $b_{\mathbb{P}, k}(s,a,V)$ for the transition estimation measures the uncertainty of the confidence set $\caB_{\mathbb{P},k}$:
\begin{align}
   b_{\mathbb{P},k}(s, a,V) = \max_{\mathbb{P}_1,\mathbb{P}_2 \in \caB_{\mathbb{P},k}} (\mathbb{P}_1-\mathbb{P}_2)V(s,a). 
\end{align}
Suppose $V_{max,k,s,a} = \argmax_{V \in \mathcal{V}} b_{\mathbb{P},k}(s, a,V)$, then we use $V_{max,t,s_{t,h,i},a_{t,h,i}}$ as the online target for the history sample $(s_{t,h,i},a_{t,h,i}, s_{t,h+1,i})$. With a slight abuse of notation, we use $b_{\mathbb{P},k}(s, a) = \max_{V \in \mathcal{V}} b_{\mathbb{P},k}(s, a,V)$ to denote the maximum uncertainty for a given state-action pair $(s,a)$.

\textbf{Near-optimal Policy Set} \ With the estimated model, algorithms for standard RL calculate the optimistic value function using Bellman backup to balance the exploration and exploitation. However, in PbRL, we cannot update the value function through Bellman update since the environment can be non-Markovian. Even worse, since we only get feedback about the preference between trajectory pairs, we cannot directly evaluate a single policy. Inspired from a recent work for PbRL in the tabular setting~\citep{pacchiano2021dueling}, we construct a near-optimal policy set using the preference information.

Define $b_{\mathbb{P},k}(\tau) = \sum_{(s,a) \in \tau} b_{\mathbb{P},k}(s,a)$. With the constructed confidence set and the bonus terms, then we construct the following set $\caS_k$:
\begin{align}\label{eqn: definition of S_k}
    \caS_k = \Big\{\pi \mid &\mathbb{E}_{\tau \sim (\hat{\mathbb{P}}_k,\pi), \tau_0 \sim (\hat{\mathbb{P}}_k,\pi_0)}\big(\hat{\mathbb{T}}_k(\tau,\tau_0) + b_{\mathbb{T},k}(\tau,\tau_0) \notag\\ 
    & + b_{\mathbb{P},k}(\tau) + b_{\mathbb{P},k}(\tau_0) \big) \geq \frac{1}{2}, \forall \pi_0 \in \Pi \Big\} .
\end{align}
Intuitively speaking, $\mathcal{S}_k$ consists of policies that there is no other policy significantly outperforms it. By executing policies in $\caS_k$ in episode $k$, We can guarantee that the regret suffered in episode $k$ to be less than the summation of the bonuses of the trajectories $(\tau_{k,1},\tau_{k,2})$, thus solve the exploitation problem in PbRL. With our concentration analysis, we can guarantee that the optimal policy $\pi^* \in \caS_k$ for any $k \in [K]$ with high probability.

\textbf{Exploratory Policies} \
Now we explain how to deal with the exploration problem in PbRL with general function approximation. Since we have already defined the uncertainty $b_{\mathbb{T},k}$ and $b_{\mathbb{P},k}$ for trajectory pairs, we can choose two policies in $\mathcal{S}_k$ that maximize the uncertainty, and thus encourage exploration:
\begin{equation} 
    \begin{aligned} \label{eqn: compute policies} 
        (\pi_{k,1},\pi_{k,2}) = &\argmax_{\pi_1,\pi_2 \in \caS_k} \mathbb{E}_{ \tau_1 \sim (\hat{\mathbb{P}}_k,\pi_1), \tau_2 \sim (\hat{\mathbb{P}}_k,\pi_2)} \\
    &\big(b_{\mathbb{T},k}(\tau_1,\tau_2) + b_{\mathbb{P},k}(\tau_1) + b_{\mathbb{P},k}(\tau_2)\big) .
    \end{aligned}
\end{equation}

\subsection{Regret Upper Bound}

\begin{theorem}
\label{theorem: dueling RL}
With probability at least $1-\delta$, the regret of Algorithm~\ref{alg: PbRL} is upper bounded by 
\begin{align*}
Reg(K) &\leq \tilde{O}( \sqrt{d_{\mathbb{P}}HK \log(\caN\left(\caF_{\mathbb{P}}, 1/K, \|\cdot\|_{\infty}\right)/\delta)} \\
& \qquad + \sqrt{d_{\mathbb{T}}K \log(\caN\left(\caF_{\mathbb{T}}, 1/K, \|\cdot\|_{\infty}\right)/\delta)}),
\end{align*}
where $d_{\mathbb{P}}$ and $d_{\mathbb{T}}$ is the $1/K$-Eluder dimension of $\caF_{\mathbb{P}}$ and $\caF_{\mathbb{T}}$, respectively.
\end{theorem}

The regret bound in Theorem~\ref{theorem: dueling RL} has polynomial dependence on the Eluder dimension of the function class $\caF_{\mathbb{P}}$ and $\caF_{\mathbb{T}}$, and has no dependence on the cardinality of the state-action space. We defer the proof of Theorem~\ref{theorem: dueling RL} to Appendix~\ref{appendix: proof of dueling RL theorem}. When specialized to linear setting, the regret of Algorithm~\ref{alg: PbRL} can be bounded by the following corollary.
\begin{corollary}[Linear Mixture Model and Linear Preference Function]
\label{theorem: upper bound in linear setting}
For the setting of linear mixture models and linear preference functions defined in Remark~\ref{remark: linear preference} and \ref{remark: linear mixture model}, the regret of Algorithm~\ref{alg: PbRL} is upper bounded by
\begin{align*}
Reg(K) &\leq {O}( \tilde{d}_{\mathbb{P}}\sqrt{HK \log(BK)\log(BK/\delta)}  \\
& \qquad + \tilde{d}_{\mathbb{T}}\sqrt{K\log(LSK)\log(LSK/\delta)}),
\end{align*}
where $\tilde{d}_{\mathbb{P}}$ and $\tilde{d}_{\mathbb{T}}$ are the feature dimension of linear mixture models and linear preference functions, respectively.
\end{corollary}

RL with once-per-episode feedback (cf. Section \ref{sec:prelim:once:per:episode}) is another paradigm introduced to tackle the problem of the lack of reward functions. Though the setting is different, we find that our Algorithm~\ref{alg: PbRL} can be almost directly applied to this setting with near-optimal regret. Since this is not the main focus of this work, we refer the interested readers to Appendix~\ref{appendix: once-per-episode feedback} for detailed algorithm and proof.
\begin{theorem}
\label{theorem: trajectory feedback, main text}
With probability at least $1-\delta$, the regret of Algorithm~\ref{alg: RL with Trajectory Feedback} for RL with once-per-episode feedback is upper bounded by
\begin{align*}
    Reg(K) \leq \tilde{O}&( \sqrt{d_{\mathbb{P}}HK \log(\caN\left(\caF_{\mathbb{P}}, 1/K, \|\cdot\|_{\infty}\right)/\delta)} \\
    &+ \sqrt{d_{\mathbb{G}}K \log(\caN\left(\caF_{\mathbb{G}}, 1/K, \|\cdot\|_{\infty}\right)/\delta)}),
\end{align*}
where $d_{\mathbb{P}}$ and $d_{\mathbb{G}}$ is the $1/K$-Eluder dimension of $\caF_{\mathbb{P}}$ and $\caF_{\mathbb{G}}$, respectively.
\end{theorem}
To the best of our knowledge, this is the first provably efficient algorithm for the problem of RL with once-per-episode feedback with general function approximation, which covers the result for RL with once-per-episode feedback in the tabular case~\citep{chatterji2021theory}. 


\subsection{Information-Theoretic Lower Bound}

In this subsection, we establish the lower bound for PbRL in the linear setting, which is derived using the reduction from the problem of RL with once-per-episode feedback.

Firstly, we show the reduction from the problem of RL with once-per-episode feedback setting to the  PbRL setting. Specifically, suppose we have an algorithm $\mathcal{ALG}$ for PbRL problems, we design a reduction protocol to solve the RL with once-per-episode feedback problem using Algorithm $\mathcal{ALG}$.

\begin{algorithm}
  \caption{Reduction Protocol}
  \label{alg: reduction protocol}
    \begin{algorithmic}[1]
        \FOR{episode $k = 1,\cdots,  K/2 $}
              \STATE Invoke Algorithm $\mathcal{ALG}$ and obtain the policy $\pi_{k,1}$ and $\pi_{k,2}$.
              \STATE Execute the policy $\pi_{k,1}$, and then observe the trajectory $\tau_{k,1}$ and the corresponding feedback $y_{k,1}$.
              \STATE Execute the policy $\pi_{k,2}$, and then observe the trajectory $\tau_{k,2}$ and the corresponding feedback $y_{k,2}$.
              \STATE Define the preference $o_k = \mathbbm{1}(y_{k,1} > y_{k,2})$ if $y_{k,1} \neq y_{k,2}$, and $o_k = \begin{cases} 1 & w.p. \  1/2 \\ 0 & w.p. \  1/2 \end{cases}$ otherwise.
              \STATE Send the information $(\tau_{k,1}, \tau_{k,2})$ and $o_k$ to Algorithm $\mathcal{ALG}$.
      \ENDFOR
    \end{algorithmic}
  \end{algorithm}

The reduction protocol is described in Algorithm~\ref{alg: reduction protocol}. In each episode, the agent invokes Algorithm $\mathcal{ALG}$ to obtain the policy $\pi_{k,1}$ and $\pi_{k,2}$ based on the history data. The agent then executes these two policies, observes the trajectory $\tau_{k,1}, \tau_{k,2}$, and receives the corresponding feedback $y_{k,1},y_{k,2}$. We define the preference $o_k = \mathbbm{1}(y_{k,1} > y_{k,2})$ if $y_{k,1} \neq y_{k,2}$, and $o_k = \begin{cases} 1 & w.p. \ 1/2 \\ 0 & w.p. \ 1/2 \end{cases}$ otherwise. Note that in our design, we have $\Pr \{o_{k} = 1\} = \frac{\Pr(y_{k,1} = 1) - \Pr(y_{k,2} = 1) + 1}{2}$. Finally, we send the information $(\tau_{k,1}, \tau_{k,2})$ and $o_k$ obtained in this episode to Algorithm $\mathcal{ALG}$. We have the following proposition based on the reduction protocol.

\begin{proposition} \label{theorem: reduction}
    Suppose there exists a PbRL algorithm $\mathcal{ALG}$ with regret $poly(K,H,d_{\mathbb{P}}, d_{\mathbb{T}}, \mathcal{N}_{\mathbb{P}} ,  \mathcal{N}_{\mathbb{T}}, 1/\delta)$. For RL with once-per-episode feedback, there exists an algorithm with regret  $poly(K,H,d_{\mathbb{P}}, d_{\caF}, \mathcal{N}_{\mathbb{P}} ,  \mathcal{N}_{\caF}, 1/\delta)$, where $d_{\caF}$ and $\mathcal{N}_{\caF}$ is the Eluder dimension and covering number of the function space 
    \begin{align*}
    \caF &= \{ f(\tau_1,\tau_2): \operatorname{Traj} \times \operatorname{Traj} \rightarrow [0,1], \\
    & \text{s.t. } f(\tau_1,\tau_2) = \frac{g(\tau_{1}) - g(\tau_2) + 1}{2}, g(\tau) \in \mathcal{F}_{\mathbb{G}} \}.
    \end{align*}
\end{proposition}



We state the lower bound for RL with once-per-episode feedback in Theorem~\ref{theorem:lb:once:per:episode}, and defer the proof to Appendix~\ref{appendix: lower bound once-per-episode feedback}.

\begin{theorem}[Lower Bound for RL with Once-per-episode Feedback] \label{theorem:lb:once:per:episode}
    For any algorithm for RL with once-per-episode feedback problem, there exists a $\tilde{d}_{\mathbb{P}}$-dimensional linear mixture MDP with a $\tilde{d}_{\mathbb{G}}$-dimensional linear preference function such that the regret incurred by this algorithms is at least $\Omega(\tilde{d}_{{\PP}}\sqrt{K} + \tilde{d}_{\mathbb{G}}\sqrt{K})$.
\end{theorem}

The following lower bound for PbRL is implied by Proposition~\ref{theorem: reduction} and Theorem~\ref{theorem:lb:once:per:episode}. This lower bound matches the upper bound in Corollary~\ref{theorem: upper bound in linear setting} w.r.t. the feature dimensions $\tilde{d}_{\mathbb{P}}, \tilde{d}_{\mathbb{T}}$ and the number of episode $K$ except for logarithmic factors, which indicates that our algorithm is near-optimal when specialized to the linear setting.
\begin{corollary}[Lower Bound for PbRL] \label{theorem:lb:dueling:rl}
    For any algorithm for PbRL, there exists a $\tilde{d}_{\mathbb{P}}$-dimensional linear mixture MDP with a $\tilde{d}_{\mathbb{T}}$-dimensional linear preference function such that the regret incurred by this algorithm is at least $\Omega(\tilde{d}_{\mathbb{P}} \sqrt{K} + \tilde{d}_{\mathbb{T}}\sqrt{K})$.
\end{corollary}

\section{RL with $n$-wise Comparisons}
In the previous section, we propose a sample-efficient algorithm with near-optimal regret for the problem of PbRL with trajectory feedback. However, this setting cannot cover some other RL situations with preference feedback. For example, in robotics, sampling new trajectories can be expensive and time-consuming compared with labeling preferences among trajectories. The human overseer may sample multiple trajectories in a distributed manner and compare all these trajectories with each other. In clinical trails, different medical treatments can be evaluated simultaneously, and the feedback is a pairwise comparison or a ranking among them. In this section, we propose a new setting called RL with $n$-wise comparisons, which is an extension of PbRL with trajectory feedback. We describe the setup and learning objective, followed by the algorithm and theoretical guarantees.

\subsection{Setup and Learning Objective}
Compared with the problem of PbRL, the main difference is that the agent needs to execute $n$ policies in each episode. Specifically, in the $k$-th episode, the agent executes $n$ policies $\{\pi_{k, i}\}_{i \in [n]}$ , and obtains $n$ trajectories $\{\tau_{k, i}\}_{i \in [n]}$. The agent receives the feedback of $n(n-1)/2$ pairwise comparisons $\{o_{k, i, j}\}_{1 \le i < j \le n}$, where $o_{k, i, j}$ is a Bernoulli random variable such that $\Pr (o_{k, i, j} = 1) = \Pr (\tau_{k, i} > \tau_{k, j})$.
Recall that we use the notations $\mathbb{T}(\tau_1, \tau_2) = \Pr(\tau_1 > \tau_2)$ and $\mathbb{T}(\pi_1,\pi_2) = \mathbb{E}_{\tau_{1} \sim (\mathbb{P}, \pi_{1}), \tau_2 \sim ( \mathbb{P}, \pi_2)}\mathbb{T}(\tau_1,\tau_{2})$. 
For such a problem, our goal is to minimize the regret, which is defined as 
$Reg(K) = \sum_{k = 1}^K\sum_{i = 1}^n \left(\mathbb{T}(\pi^*, \pi_{k, i}) - \frac{1}{2} \right)$,
where $\pi^*$ is defined in Assumption~\ref{assumption: definition of pistar}. When $n = 2$, this regret reduces to the regret in the standard PbRL setting.

\subsection{Algorithm}

The algorithm, which is formally described in Algorithm~\ref{alg:pairwise}, shares the similar framework with Algorithm~\ref{alg: PbRL}. The main difference is on the construction of confidence sets and bonus terms. Similar to the PbRL setting, we use least-squares regression to estimate the preference function $\mathbb{T}$:

\begin{align}
\label{eqn: definition of hatT 2}
    \hat{\mathbb{T}}_k = \argmin_{\mathbb{T}^{\prime} \in \caF_{T}} \sum_{t=1}^{k-1}\sum_{i = 1}^n\sum_{j = i+1}^n \left({\mathbb{T}}^{\prime}({\tau}_{t,i},{\tau}_{t,j})- o_{t, i, j}\right)^2.
\end{align}
Notably, we have $n(n-1)/2$ samples instead of one sample in each episode.We construct the confidence set centered at $\hat{\mathbb{T}}_k$ by 
\begin{align}
\label{eqn: confidence set for T 2}
    \caB_{\mathbb{T},k} = \left\{ \tilde{\mathbb{T}} \mid  \sum_{t=1}^{k-1}\sum_{i = 1}^n\sum_{j = i+1}^n \left(\hat{\mathbb{T}}_{k} - \tilde{\mathbb{T}}\right)^2({\tau}_{t,i}, {\tau}_{t,j}) \leq \beta_{\mathbb{T}}\right\}.
\end{align}
Given the confidence set $\caB_{\mathbb{T},k}$, we also use the function $b_{\mathbb{T},k}(\tau_1,\tau_2) = \max_{f_1, f_2 \in \caB_{\mathbb{T},k}}\left|f_1(\tau_1,\tau_2) - f_2(\tau_1,\tau_2)\right|$ to measure its uncertainty.

\begin{algorithm}[t]
  \caption{PbOP+: Pairwise Preference-based Optimistic Planning}
    \begin{algorithmic}[1] \label{alg:pairwise}
        \STATE Set $\beta_{\mathbb{T}} =  8 \log(2K \mathcal{N}\left(\mathcal{F}_{\mathbb{T}}, 1/(Kn^2),\|\cdot\|_{\infty}\right)/ \delta)$ and $\beta_{\mathbb{P}} = 8 \log(2K \mathcal{N}\left(\mathcal{F}_{\mathbb{P}}, 1/(Kn),\|\cdot\|_{\infty}\right)/ \delta)$
        \FOR{episode $k = 1,\cdots, K$}
              \STATE Calculate the estimation $\hat{\mathbb{T}}_{k}$ and $\hat{\mathbb{P}}_k$ using least-squares regression (Eqn.~\eqref{eqn: definition of hatT 2} and \eqref{eqn: definition of hatP 2}) 
              \STATE Construct the high-confidence set $\caB_{\mathbb{T},k}$ for the preference $\mathbb{T}$ (Eqn.~\eqref{eqn: confidence set for T 2})
              \STATE Construct the high-confidence set $\caB_{\mathbb{P},k}$ for transition $\mathbb{P}$ (Eqn.~\eqref{eqn: confidence set for P 2})
              \STATE Define the bonus term $b_{\mathbb{T},k}(\tau_1,\tau_2) = \max_{f_1, f_2 \in \caB_{\mathbb{T},k}}\left|f_1(\tau_1,\tau_2) - f_2(\tau_1,\tau_2)\right|$
              \STATE Define the bonus term $b_{\mathbb{P},k}(s, a) = \max_{\mathbb{P}_1,\mathbb{P}_2 \in \caB_{\mathbb{P},k}} \max_{V \in \caV} (\mathbb{P}_1-\mathbb{P}_2)V(s,a)$
              \STATE Set $b_{\mathbb{P},k}(\tau) = \sum_{(s,a) \in \tau} b_{\mathbb{P},k}(s, a)$
              \STATE Compute the policy set $\caS_k$ as Eqn.~\eqref{eqn: definition of S_k 2}
              \STATE Compute policy $(\pi_{k,1},\pi_{k,2}, \cdots, \pi_{k, n})$ as Eqn.~\eqref{eqn: compute policies 2}
              \STATE Execute the policies $(\pi_{k,1}, \pi_{k, 2}, \cdots, \pi_{k,n})$ for one episode, respectively, and then observe the trajectories $(\tau_{k,1}, \tau_{k,2}, \cdots, \pi_{k, n})$
              \STATE Receive the preference $o_{k, i, j}$ between $\tau_{k,i}$ and $\tau_{k,j}$ for all $(i, j) \in \{(i, j) | 1 \le i < j \le n\}$
      \ENDFOR
    \end{algorithmic}
  \end{algorithm}

For estimating the transition dynamics, we utilize historical trajectories $\{\tau_{t, i}\}_{(t, i) \in [k - 1] \times [n]}$ to perform the least-squares regression: 
\begin{align}
    \label{eqn: definition of hatP 2}
    \hat{\mathbb{P}}_{k} =\operatorname{argmin}_{\mathbb{P}^{\prime} \in \mathcal{P}} \sum_{i=1}^{n} \sum_{t=1}^{k-1} &\sum_{h=1}^{H} \big(\langle \mathbb{P}^{\prime}(\cdot \mid s_{t,h,i}, a_{t,h,i}), V_{k,h,i}\rangle \quad\notag\\
    &\quad -V_{k,h,i}(s_{k,h+1,i})\big)^{2},
\end{align}
where $V_{k, h, i}$ is the target function defined as follows. Specifically, we construct the high confidence set for transition $\mathbb{P}$, which is defined as
\begin{align}
\label{eqn: confidence set for P 2}
    \caB_{\mathbb{P},k} = \left\{ \mathbb{P}' \mid L_k(\mathbb{P}', \hat{\mathbb{P}}_k) \leq \beta_{\mathbb{P}}\right\},
\end{align}
where $L_k(\cdot, \cdot)$ is defined by
\begin{align}
    L_{k}\left(\mathbb{P}_1, {\mathbb{P}}_{2}\right)=\sum_{i=1}^n &\sum_{t=1}^{k-1} \sum_{h=1}^{H}\big(\langle \mathbb{P}_1\left(\cdot \mid s_{t,h,i}, a_{t,h,i}\right) \notag\\
    & -{\mathbb{P}}_{2}\left(\cdot \mid s_{t,h,i}, a_{t,h,i}\right), V_{t,h,i}\rangle\big)^{2}.
\end{align}
Given any $V \in \mathcal{V}$, we choose the associated bonus $b_{\mathbb{P}, k}(s, a, V)$ as
$b_{\mathbb{P}, k}(s, a, V) = \max_{\mathbb{P}_1, \mathbb{P}_2}(\mathbb{P}_1 - \mathbb{P}_2)V(s, a).$
Such a bonus function measures the uncertainty of the confidence set $\mathcal{B}_{\mathbb{P},k}$. We also choose the target value function $V_{t, h, i}$ as the function that can maximize the uncertainty. Formally, let $V_{\max, k, s, a} = \argmax_{V \in \mathcal{V}} b_{\mathbb{P}, k}(s, a, V)$, then we use $V_{\max, t, s_{t, h, i}, a_{t, h, i}}$ as the online target for the historical sample $(s_{t, h, i}, a_{t, h, i}, s_{t, h+1, i})$. Also we denote $b_{\mathbb{P}, k}(s, a) = \max_{V \in \mathcal{V}}b_{\mathbb{P}, k}(s, a, V)$ and $b_{\mathbb{P}, k}(\tau) = \sum_{(s, a) \in \tau}b_{\mathbb{P}, k}(s, a)$.

Given the estimated transition $\hat{\mathbb{P}}_k$, bonus for transition $b_{\mathbb{P}, k}$ and bonus for preference function $b_{\mathbb{T}, k}$, we can construct the near-optimal set $\mathcal{S}_k$ like Eqn.~\eqref{eqn: definition of S_k}:
\begin{align}\label{eqn: definition of S_k 2}
    \caS_k = \Big\{\pi \mid &\mathbb{E}_{\tau \sim (\hat{\mathbb{P}}_k,\pi), \tau_0 \sim (\hat{\mathbb{P}}_k,\pi_0)}\big(\hat{\mathbb{T}}_k(\tau,\tau_0) + b_{\mathbb{T},k}(\tau,\tau_0) \notag\\ 
    & + b_{\mathbb{P},k}(\tau) + b_{\mathbb{P},k}(\tau_0) \big) \geq \frac{1}{2}, \forall \pi_0 \in \Pi \Big\} .
\end{align}
Finally, we choose the exploratory polices $(\pi_{k,1},\pi_{k,2}, \cdots, \pi_{k, n})$ that can maximize the pairwise uncertainty. We luckily find that the summation of bonuses exactly characterize the uncertainty of the policy tuple:
\begin{align} \label{eqn: compute policies 2}
    &(\pi_{k,1},\pi_{k,2}, \cdots, \pi_{k, n}) = \argmax_{\pi_1,\pi_2, \cdots, \pi_n \in \caS_k} \sum_{i = 1}^n \sum_{j = i+1}^{n}  \\
    & \mathbb{E}_{ \tau_i \sim (\hat{\mathbb{P}}_k,\pi_i), \tau_j \sim (\hat{\mathbb{P}}_k,\pi_j)} \left(b_{\mathbb{T},k}(\tau_i,\tau_j) + b_{\mathbb{P},k}(\tau_i) + b_{\mathbb{P},k}(\tau_j)\right). \notag 
\end{align}

\subsection{Theoretical Guarantees}

In the following theorem, we establish the regret upper bound for Algorithm~\ref{alg:pairwise}. The proof of this theorem is deferred to Appendix~\ref{appendix:pairwise:comparison}.

\begin{theorem}
\label{theorem: pairwise}
With probability at least $1-\delta$, the regret of Algorithm~\ref{alg:pairwise} is upper bounded by 
\begin{align*}
Reg(K) &\leq \tilde{O}( \sqrt{d_{\mathbb{P}}HnK \log(\caN\left(\caF_{\mathbb{P}}, 1/(Kn), \|\cdot\|_{\infty}\right)/\delta)} \\
& \qquad + \sqrt{d_{\mathbb{T}}K \log(\caN\left(\caF_{\mathbb{T}}, 1/(Kn^2), \|\cdot\|_{\infty}\right)/\delta)}),
\end{align*}
where $d_{\mathbb{P}}$ and $d_{\mathbb{T}}$ is the $1/K$-Eluder dimension of $\caF_{\mathbb{P}}$ and $\caF_{\mathbb{T}}$, respectively. 
\end{theorem}

By replacing $K$ with $K/n$ in Theorem~\ref{theorem: pairwise}, we obtain a  bound of $\tilde{O}( \sqrt{d_{\mathbb{P}}HK \log(\caN\left(\caF_{\mathbb{P}}, 1/(Kn), \|\cdot\|_{\infty}\right)/\delta)} + \sqrt{d_{\mathbb{T}}K/n \log(\caN\left(\caF_{\mathbb{T}}, 1/(Kn^2), \|\cdot\|_{\infty}\right)/\delta)})$. This improves the regret bound in Theorem~\ref{theorem: dueling RL} by a factor of $\sqrt{n}$ on the second term, which is the benefit of additional information from $n$-wise comparisons.
\section{Conclusion}
\label{sec: conlusion}
This paper studies the regret minimization problem of PbRL with trajectory feedback and general function approximation. Based on the value-targeted regression and optimistic planning methods, we propose a novel RL algorithm called PbOP with regret $\tilde{O}(\operatorname{poly}(dH)\sqrt{K})$. Our lower bound indicates that our regret upper bound is tight w.r.t. the feature dimension and the number of episodes when specialized to the linear setting.  Furthermore, we formulate a novel setting called RL with $n$-wise comparisons and provide the first sample efficient algorithm in this setting.

A few problems still remain open. Firstly, there is still a gap of $\sqrt{H}$ between Corollaries~\ref{theorem: upper bound in linear setting} and \ref{theorem:lb:dueling:rl}. We conjecture that our upper bound is not tight, which can possibly be improved with more refined concentration analysis based on Bernstein bounds. Secondly, our algorithm is computationally inefficient due to non-Markovian feedback. It is tempting to design both statistically and computationally efficient algorithms in a relaxed PbRL setting. Finally, our setting of RL with $n$-wise comparisons does not cover the case where the feedback among $n$ trajectories is a $n$-wise ranking~\citep{negahban2018learning}, which is also an interesting problem to be explored in future research.

\section*{Acknowledgement}
Liwei Wang was supported by National Key R\&D Program of China (2018YFB1402600), Exploratory Research Project of Zhejiang Lab (No. 2022RC0AN02), BJNSF (L172037), Project 2020BD006 supported by PKUBaidu Fund, the major key project of PCL (PCL2021A12).

\bibliography{references}

\begin{thebibliography}{52}
\providecommand{\natexlab}[1]{#1}
\providecommand{\url}[1]{\texttt{#1}}
\expandafter\ifx\csname urlstyle\endcsname\relax
  \providecommand{\doi}[1]{doi: #1}\else
  \providecommand{\doi}{doi: \begingroup \urlstyle{rm}\Url}\fi

\bibitem[Abbeel \& Ng(2004)Abbeel and Ng]{abbeel2004apprenticeship}
Abbeel, P. and Ng, A.~Y.
\newblock Apprenticeship learning via inverse reinforcement learning.
\newblock In \emph{Proceedings of the twenty-first international conference on
  Machine learning}, pp.\ ~1, 2004.

\bibitem[Amodei et~al.(2016)Amodei, Olah, Steinhardt, Christiano, Schulman, and
  Man{\'e}]{amodei2016concrete}
Amodei, D., Olah, C., Steinhardt, J., Christiano, P., Schulman, J., and
  Man{\'e}, D.
\newblock Concrete problems in ai safety.
\newblock \emph{arXiv preprint arXiv:1606.06565}, 2016.

\bibitem[Ayoub et~al.(2020)Ayoub, Jia, Szepesvari, Wang, and
  Yang]{ayoub2020model}
Ayoub, A., Jia, Z., Szepesvari, C., Wang, M., and Yang, L.
\newblock Model-based reinforcement learning with value-targeted regression.
\newblock In \emph{International Conference on Machine Learning}, pp.\
  463--474. PMLR, 2020.

\bibitem[Azar et~al.(2017)Azar, Osband, and Munos]{azar2017minimax}
Azar, M.~G., Osband, I., and Munos, R.
\newblock Minimax regret bounds for reinforcement learning.
\newblock In \emph{International Conference on Machine Learning}, pp.\
  263--272. PMLR, 2017.

\bibitem[Berkenkamp et~al.(2021)Berkenkamp, Krause, and
  Schoellig]{berkenkamp2021bayesian}
Berkenkamp, F., Krause, A., and Schoellig, A.~P.
\newblock Bayesian optimization with safety constraints: safe and automatic
  parameter tuning in robotics.
\newblock \emph{Machine Learning}, pp.\  1--35, 2021.

\bibitem[Busa-Fekete et~al.(2014)Busa-Fekete, Sz{\"o}r{\'e}nyi, Weng, Cheng,
  and H{\"u}llermeier]{busa2014preference}
Busa-Fekete, R., Sz{\"o}r{\'e}nyi, B., Weng, P., Cheng, W., and
  H{\"u}llermeier, E.
\newblock Preference-based reinforcement learning: evolutionary direct policy
  search using a preference-based racing algorithm.
\newblock \emph{Machine Learning}, 97\penalty0 (3):\penalty0 327--351, 2014.

\bibitem[Busa-Fekete et~al.(2018)Busa-Fekete, H{\"u}llermeier, and
  Mesaoudi-Paul]{busa2018preference}
Busa-Fekete, R., H{\"u}llermeier, E., and Mesaoudi-Paul, A.~E.
\newblock Preference-based online learning with dueling bandits: A survey.
\newblock \emph{arXiv preprint arXiv:1807.11398}, 2018.

\bibitem[Cai et~al.(2020)Cai, Yang, Jin, and Wang]{cai2020provably}
Cai, Q., Yang, Z., Jin, C., and Wang, Z.
\newblock Provably efficient exploration in policy optimization.
\newblock In \emph{International Conference on Machine Learning}, pp.\
  1283--1294. PMLR, 2020.

\bibitem[Chatterji et~al.(2021)Chatterji, Pacchiano, Bartlett, and
  Jordan]{chatterji2021theory}
Chatterji, N.~S., Pacchiano, A., Bartlett, P.~L., and Jordan, M.~I.
\newblock On the theory of reinforcement learning with once-per-episode
  feedback.
\newblock \emph{arXiv preprint arXiv:2105.14363}, 2021.

\bibitem[Chen et~al.(2021{\natexlab{a}})Chen, Hu, Jin, Li, and
  Wang]{chen2021understanding}
Chen, X., Hu, J., Jin, C., Li, L., and Wang, L.
\newblock Understanding domain randomization for sim-to-real transfer.
\newblock \emph{arXiv preprint arXiv:2110.03239}, 2021{\natexlab{a}}.

\bibitem[Chen et~al.(2021{\natexlab{b}})Chen, Hu, Yang, and Wang]{chen2021near}
Chen, X., Hu, J., Yang, L.~F., and Wang, L.
\newblock Near-optimal reward-free exploration for linear mixture mdps with
  plug-in solver.
\newblock \emph{arXiv preprint arXiv:2110.03244}, 2021{\natexlab{b}}.

\bibitem[Christiano et~al.(2017)Christiano, Leike, Brown, Martic, Legg, and
  Amodei]{christiano2017deep}
Christiano, P., Leike, J., Brown, T.~B., Martic, M., Legg, S., and Amodei, D.
\newblock Deep reinforcement learning from human preferences.
\newblock \emph{arXiv preprint arXiv:1706.03741}, 2017.

\bibitem[Du et~al.(2019)Du, Luo, Wang, and Zhang]{du2019provably}
Du, S.~S., Luo, Y., Wang, R., and Zhang, H.
\newblock Provably efficient $ q $-learning with function approximation via
  distribution shift error checking oracle.
\newblock \emph{arXiv preprint arXiv:1906.06321}, 2019.

\bibitem[Du et~al.(2021)Du, Kakade, Lee, Lovett, Mahajan, Sun, and
  Wang]{du2021bilinear}
Du, S.~S., Kakade, S.~M., Lee, J.~D., Lovett, S., Mahajan, G., Sun, W., and
  Wang, R.
\newblock Bilinear classes: A structural framework for provable generalization
  in rl.
\newblock \emph{arXiv preprint arXiv:2103.10897}, 2021.

\bibitem[Efroni et~al.(2020)Efroni, Merlis, and
  Mannor]{efroni2020reinforcement}
Efroni, Y., Merlis, N., and Mannor, S.
\newblock Reinforcement learning with trajectory feedback.
\newblock \emph{arXiv preprint arXiv:2008.06036}, 2020.

\bibitem[Falahatgar et~al.(2017{\natexlab{a}})Falahatgar, Hao, Orlitsky,
  Pichapati, and Ravindrakumar]{falahatgar2017maxing}
Falahatgar, M., Hao, Y., Orlitsky, A., Pichapati, V., and Ravindrakumar, V.
\newblock Maxing and ranking with few assumptions.
\newblock In \emph{Proceedings of the 31st International Conference on Neural
  Information Processing Systems}, pp.\  7063--7073, 2017{\natexlab{a}}.

\bibitem[Falahatgar et~al.(2017{\natexlab{b}})Falahatgar, Orlitsky, Pichapati,
  and Suresh]{falahatgar2017maximum}
Falahatgar, M., Orlitsky, A., Pichapati, V., and Suresh, A.~T.
\newblock Maximum selection and ranking under noisy comparisons.
\newblock In \emph{International Conference on Machine Learning}, pp.\
  1088--1096. PMLR, 2017{\natexlab{b}}.

\bibitem[Foster et~al.(2020)Foster, Rakhlin, Simchi-Levi, and
  Xu]{foster2020instance}
Foster, D.~J., Rakhlin, A., Simchi-Levi, D., and Xu, Y.
\newblock Instance-dependent complexity of contextual bandits and reinforcement
  learning: A disagreement-based perspective.
\newblock \emph{arXiv preprint arXiv:2010.03104}, 2020.

\bibitem[Ho \& Ermon(2016)Ho and Ermon]{ho2016generative}
Ho, J. and Ermon, S.
\newblock Generative adversarial imitation learning.
\newblock \emph{Advances in neural information processing systems},
  29:\penalty0 4565--4573, 2016.

\bibitem[Hussein et~al.(2017)Hussein, Gaber, Elyan, and
  Jayne]{hussein2017imitation}
Hussein, A., Gaber, M.~M., Elyan, E., and Jayne, C.
\newblock Imitation learning: A survey of learning methods.
\newblock \emph{ACM Computing Surveys (CSUR)}, 50\penalty0 (2):\penalty0 1--35,
  2017.

\bibitem[Jain et~al.(2013)Jain, Wojcik, Joachims, and Saxena]{jain2013learning}
Jain, A., Wojcik, B., Joachims, T., and Saxena, A.
\newblock Learning trajectory preferences for manipulators via iterative
  improvement.
\newblock \emph{arXiv preprint arXiv:1306.6294}, 2013.

\bibitem[Jain et~al.(2015)Jain, Sharma, Joachims, and Saxena]{jain2015learning}
Jain, A., Sharma, S., Joachims, T., and Saxena, A.
\newblock Learning preferences for manipulation tasks from online coactive
  feedback.
\newblock \emph{The International Journal of Robotics Research}, 34\penalty0
  (10):\penalty0 1296--1313, 2015.

\bibitem[Jiang et~al.(2017)Jiang, Krishnamurthy, Agarwal, Langford, and
  Schapire]{jiang2017contextual}
Jiang, N., Krishnamurthy, A., Agarwal, A., Langford, J., and Schapire, R.~E.
\newblock Contextual decision processes with low bellman rank are
  pac-learnable.
\newblock In \emph{International Conference on Machine Learning}, pp.\
  1704--1713. PMLR, 2017.

\bibitem[Jin et~al.(2020)Jin, Yang, Wang, and Jordan]{jin2020provably}
Jin, C., Yang, Z., Wang, Z., and Jordan, M.~I.
\newblock Provably efficient reinforcement learning with linear function
  approximation.
\newblock In \emph{Conference on Learning Theory}, pp.\  2137--2143. PMLR,
  2020.

\bibitem[Jin et~al.(2021)Jin, Liu, and Miryoosefi]{jin2021bellman}
Jin, C., Liu, Q., and Miryoosefi, S.
\newblock Bellman eluder dimension: New rich classes of rl problems, and
  sample-efficient algorithms.
\newblock \emph{arXiv preprint arXiv:2102.00815}, 2021.

\bibitem[Kong et~al.(2021)Kong, Salakhutdinov, Wang, and Yang]{kong2021online}
Kong, D., Salakhutdinov, R., Wang, R., and Yang, L.~F.
\newblock Online sub-sampling for reinforcement learning with general function
  approximation.
\newblock \emph{arXiv preprint arXiv:2106.07203}, 2021.

\bibitem[Mnih et~al.(2013)Mnih, Kavukcuoglu, Silver, Graves, Antonoglou,
  Wierstra, and Riedmiller]{mnih2013playing}
Mnih, V., Kavukcuoglu, K., Silver, D., Graves, A., Antonoglou, I., Wierstra,
  D., and Riedmiller, M.
\newblock Playing atari with deep reinforcement learning.
\newblock \emph{arXiv preprint arXiv:1312.5602}, 2013.

\bibitem[Negahban et~al.(2018)Negahban, Oh, Thekumparampil, and
  Xu]{negahban2018learning}
Negahban, S., Oh, S., Thekumparampil, K.~K., and Xu, J.
\newblock Learning from comparisons and choices.
\newblock \emph{The Journal of Machine Learning Research}, 19\penalty0
  (1):\penalty0 1478--1572, 2018.

\bibitem[Ng et~al.(1999)Ng, Harada, and Russell]{ng1999policy}
Ng, A.~Y., Harada, D., and Russell, S.
\newblock Policy invariance under reward transformations: Theory and
  application to reward shaping.
\newblock In \emph{Icml}, volume~99, pp.\  278--287, 1999.

\bibitem[Ng et~al.(2000)Ng, Russell, et~al.]{ng2000algorithms}
Ng, A.~Y., Russell, S.~J., et~al.
\newblock Algorithms for inverse reinforcement learning.
\newblock In \emph{Icml}, volume~1, pp.\ ~2, 2000.

\bibitem[Novoseller et~al.(2020)Novoseller, Wei, Sui, Yue, and
  Burdick]{novoseller2020dueling}
Novoseller, E., Wei, Y., Sui, Y., Yue, Y., and Burdick, J.
\newblock Dueling posterior sampling for preference-based reinforcement
  learning.
\newblock In \emph{Conference on Uncertainty in Artificial Intelligence}, pp.\
  1029--1038. PMLR, 2020.

\bibitem[Osa et~al.(2018)Osa, Pajarinen, Neumann, Bagnell, Abbeel, and
  Peters]{osa2018algorithmic}
Osa, T., Pajarinen, J., Neumann, G., Bagnell, J.~A., Abbeel, P., and Peters, J.
\newblock An algorithmic perspective on imitation learning.
\newblock \emph{arXiv preprint arXiv:1811.06711}, 2018.

\bibitem[Osband \& Van~Roy(2014)Osband and Van~Roy]{osband2014model}
Osband, I. and Van~Roy, B.
\newblock Model-based reinforcement learning and the eluder dimension.
\newblock \emph{arXiv preprint arXiv:1406.1853}, 2014.

\bibitem[Pacchiano et~al.(2021)Pacchiano, Saha, and Lee]{pacchiano2021dueling}
Pacchiano, A., Saha, A., and Lee, J.
\newblock Dueling rl: Reinforcement learning with trajectory preferences.
\newblock \emph{arXiv preprint arXiv:2111.04850}, 2021.

\bibitem[Russo \& Van~Roy(2013)Russo and Van~Roy]{russo2013eluder}
Russo, D. and Van~Roy, B.
\newblock Eluder dimension and the sample complexity of optimistic exploration.
\newblock In \emph{NIPS}, pp.\  2256--2264. Citeseer, 2013.

\bibitem[Russo \& Van~Roy(2014)Russo and Van~Roy]{russo2014learning}
Russo, D. and Van~Roy, B.
\newblock Learning to optimize via posterior sampling.
\newblock \emph{Mathematics of Operations Research}, 39\penalty0 (4):\penalty0
  1221--1243, 2014.

\bibitem[Silver et~al.(2017)Silver, Schrittwieser, Simonyan, Antonoglou, Huang,
  Guez, Hubert, Baker, Lai, Bolton, et~al.]{silver2017mastering}
Silver, D., Schrittwieser, J., Simonyan, K., Antonoglou, I., Huang, A., Guez,
  A., Hubert, T., Baker, L., Lai, M., Bolton, A., et~al.
\newblock Mastering the game of go without human knowledge.
\newblock \emph{nature}, 550\penalty0 (7676):\penalty0 354--359, 2017.

\bibitem[Vinyals et~al.(2019)Vinyals, Babuschkin, Czarnecki, Mathieu, Dudzik,
  Chung, Choi, Powell, Ewalds, Georgiev, et~al.]{vinyals2019grandmaster}
Vinyals, O., Babuschkin, I., Czarnecki, W.~M., Mathieu, M., Dudzik, A., Chung,
  J., Choi, D.~H., Powell, R., Ewalds, T., Georgiev, P., et~al.
\newblock Grandmaster level in starcraft ii using multi-agent reinforcement
  learning.
\newblock \emph{Nature}, 575\penalty0 (7782):\penalty0 350--354, 2019.

\bibitem[Wang et~al.(2020{\natexlab{a}})Wang, Du, Yang, and
  Salakhutdinov]{wang2020reward}
Wang, R., Du, S.~S., Yang, L.~F., and Salakhutdinov, R.
\newblock On reward-free reinforcement learning with linear function
  approximation.
\newblock \emph{arXiv preprint arXiv:2006.11274}, 2020{\natexlab{a}}.

\bibitem[Wang et~al.(2020{\natexlab{b}})Wang, Salakhutdinov, and
  Yang]{wang2020reinforcement}
Wang, R., Salakhutdinov, R., and Yang, L.~F.
\newblock Reinforcement learning with general value function approximation:
  Provably efficient approach via bounded eluder dimension.
\newblock \emph{arXiv preprint arXiv:2005.10804}, 2020{\natexlab{b}}.

\bibitem[Wirth \& F{\"u}rnkranz(2012)Wirth and F{\"u}rnkranz]{wirth2012first}
Wirth, C. and F{\"u}rnkranz, J.
\newblock First steps towards learning from game annotations.
\newblock 2012.

\bibitem[Wirth \& F{\"u}rnkranz(2014)Wirth and
  F{\"u}rnkranz]{wirth2014learning}
Wirth, C. and F{\"u}rnkranz, J.
\newblock On learning from game annotations.
\newblock \emph{IEEE Transactions on Computational Intelligence and AI in
  Games}, 7\penalty0 (3):\penalty0 304--316, 2014.

\bibitem[Wirth et~al.(2017)Wirth, Akrour, Neumann, F{\"u}rnkranz,
  et~al.]{wirth2017survey}
Wirth, C., Akrour, R., Neumann, G., F{\"u}rnkranz, J., et~al.
\newblock A survey of preference-based reinforcement learning methods.
\newblock \emph{Journal of Machine Learning Research}, 18\penalty0
  (136):\penalty0 1--46, 2017.

\bibitem[Xu et~al.(2020{\natexlab{a}})Xu, Joshi, Singh, and
  Dubrawski]{xu2020zeroth}
Xu, Y., Joshi, A., Singh, A., and Dubrawski, A.
\newblock Zeroth order non-convex optimization with dueling-choice bandits.
\newblock In \emph{Conference on Uncertainty in Artificial Intelligence}, pp.\
  899--908. PMLR, 2020{\natexlab{a}}.

\bibitem[Xu et~al.(2020{\natexlab{b}})Xu, Wang, Yang, Singh, and
  Dubrawski]{xu2020preference}
Xu, Y., Wang, R., Yang, L.~F., Singh, A., and Dubrawski, A.
\newblock Preference-based reinforcement learning with finite-time guarantees.
\newblock \emph{arXiv preprint arXiv:2006.08910}, 2020{\natexlab{b}}.

\bibitem[Yang \& Wang(2019)Yang and Wang]{yang2019sample}
Yang, L. and Wang, M.
\newblock Sample-optimal parametric q-learning using linearly additive
  features.
\newblock In \emph{International Conference on Machine Learning}, pp.\
  6995--7004. PMLR, 2019.

\bibitem[Yue et~al.(2012)Yue, Broder, Kleinberg, and Joachims]{yue2012k}
Yue, Y., Broder, J., Kleinberg, R., and Joachims, T.
\newblock The k-armed dueling bandits problem.
\newblock \emph{Journal of Computer and System Sciences}, 78\penalty0
  (5):\penalty0 1538--1556, 2012.

\bibitem[Zanette et~al.(2020)Zanette, Lazaric, Kochenderfer, and
  Brunskill]{zanette2020learning}
Zanette, A., Lazaric, A., Kochenderfer, M., and Brunskill, E.
\newblock Learning near optimal policies with low inherent bellman error.
\newblock In \emph{International Conference on Machine Learning}, pp.\
  10978--10989. PMLR, 2020.

\bibitem[Zhang et~al.(2021)Zhang, Zhou, and Gu]{zhang2021reward}
Zhang, W., Zhou, D., and Gu, Q.
\newblock Reward-free model-based reinforcement learning with linear function
  approximation.
\newblock \emph{Advances in Neural Information Processing Systems}, 34, 2021.

\bibitem[Zhao et~al.(2011)Zhao, Zeng, Socinski, and
  Kosorok]{zhao2011reinforcement}
Zhao, Y., Zeng, D., Socinski, M.~A., and Kosorok, M.~R.
\newblock Reinforcement learning strategies for clinical trials in nonsmall
  cell lung cancer.
\newblock \emph{Biometrics}, 67\penalty0 (4):\penalty0 1422--1433, 2011.

\bibitem[Zhou et~al.(2021{\natexlab{a}})Zhou, Gu, and
  Szepesvari]{zhou2021nearly}
Zhou, D., Gu, Q., and Szepesvari, C.
\newblock Nearly minimax optimal reinforcement learning for linear mixture
  markov decision processes.
\newblock In \emph{Conference on Learning Theory}, pp.\  4532--4576. PMLR,
  2021{\natexlab{a}}.

\bibitem[Zhou et~al.(2021{\natexlab{b}})Zhou, He, and Gu]{zhou2021provably}
Zhou, D., He, J., and Gu, Q.
\newblock Provably efficient reinforcement learning for discounted mdps with
  feature mapping.
\newblock In \emph{International Conference on Machine Learning}, pp.\
  12793--12802. PMLR, 2021{\natexlab{b}}.

\end{thebibliography}
\bibliographystyle{icml2022}

\newpage
\appendix
\onecolumn
\section{Proof of Theorem~\ref{theorem: dueling RL}}
\label{appendix: proof of dueling RL theorem}

\begin{lemma}
\label{lemma: high confidence bound for hatT}
Fix $\delta \in (0,1)$, with probability at least $1-\delta$, for all $k \in [K]$,
\begin{align}
    \sum_{t=1}^{k-1}\left(\hat{\mathbb{T}}_{k} - \mathbb{T}\right)^2(\tau_{t,1}, \tau_{t,2}) \leq  \beta_{\mathbb{T}}.
\end{align}
\end{lemma}

\begin{proof}
This lemma can be proved by the direct application of Lemma~\ref{lemma: auxiliary Lemma for confidence set}.
\end{proof}

By Lemma~\ref{lemma: high confidence bound for hatT}, we know that the true preference $\mathbb{T}(\tau_1,\tau_2) \in \caB_{\mathbb{T},k}$ with high probability.

\begin{lemma}
\label{lemma: high confidence bound for hatP}
Fix $\delta \in (0,1)$, with probability at least $1-\delta$, for all $k \in [K]$,
\begin{align}
    L_{k}\left(\mathbb{P}, \hat{\mathbb{P}}_{k}\right)=\sum_{i=1}^2\sum_{t=1}^{k-1} \sum_{h=1}^{H}\left(\left\langle \mathbb{P}\left(\cdot \mid s_{t,h,i}, a_{t,h,i}\right)-\hat{\mathbb{P}}_{k}\left(\cdot \mid s_{t,h,i}, a_{t,h,i}\right), V_{t,h,i}\right\rangle\right)^{2} \leq \beta_{\mathbb{P}}.
\end{align}
\end{lemma}
\begin{proof}
This lemma can be proved by the direct application of Lemma~\ref{lemma: auxiliary Lemma for confidence set}.
\end{proof}

By Lemma~\ref{lemma: high confidence bound for hatP}, we know that the true transition kernel $\mathbb{P}(s'|s,a) \in \caB_{\mathbb{P},k}$ with high probability. 

We denote the high-probability event in Lemmas~\ref{lemma: high confidence bound for hatT} and \ref{lemma: high confidence bound for hatP} as $\mathcal{E}$.

\begin{lemma}
\label{lemma: error due to transition estimation}
Under event $\caE$, for any two policies $\pi_1, \pi_2$ and scalar function $f: \operatorname{Traj} \times \operatorname{Traj} \rightarrow [0,1]$, 
\begin{align}
    \mathbb{E}_{\tau_1 \sim (\mathbb{P}, \pi_1),\tau_2 \sim (\mathbb{P}, \pi_2)}[f(\tau_1,\tau_2)] - \mathbb{E}_{\tau_1 \sim (\hat{\mathbb{P}}_{k}, \pi_1),\tau_2 \sim (\hat{\mathbb{P}}_{k}, \pi_2)}[f(\tau_1,\tau_2)] \leq \mathbb{E}_{\tau_1 \sim (\hat{\mathbb{P}}_{k}, \pi_1),\tau_2 \sim (\hat{\mathbb{P}}_{k}, \pi_2)}[b_{\mathbb{P},k}(\tau_1) + b_{\mathbb{P},k}(\tau_2)]
\end{align}
\end{lemma}
\begin{proof}
We use $\tau \sim (\hat{\mathbb{P}}_k, \pi)_h$ to denote that the first $h$ steps in the trajectory $\tau$ is sampled using policy $\pi$ from the MDP with transition $\hat{\mathbb{P}}_k$, and the state-action pairs from step $h+1$ up until the last step is sampled using policy $\pi$ from the MDP with the true transition ${\mathbb{P}}$. Therefore,
\begin{align}
    & \mathbb{E}_{\tau_1 \sim (\mathbb{P}, \pi_1),\tau_2 \sim (\mathbb{P}, \pi_2)}[f(\tau_1,\tau_2)] - \mathbb{E}_{\tau_1 \sim (\hat{\mathbb{P}}_{k}, \pi_1),\tau_2 \sim (\hat{\mathbb{P}}_{k}, \pi_2)}[f(\tau_1,\tau_2)]  \\
    = &\mathbb{E}_{\tau_1 \sim (\hat{\mathbb{P}}_k, \pi_1)_0,\tau_2 \sim (\hat{\mathbb{P}}_k, \pi_2)_0}[f(\tau_1,\tau_2)] - \mathbb{E}_{\tau_1 \sim (\hat{\mathbb{P}}_{k}, \pi_1)_{H},\tau_2 \sim (\hat{\mathbb{P}}_{k}, \pi_2)_H}[f(\tau_1,\tau_2)] \\
    = & \sum_{h=1}^{H}\mathbb{E}_{\tau_1 \sim (\hat{\mathbb{P}}_k, \pi_1)_{h-1},\tau_2 \sim (\hat{\mathbb{P}}_k, \pi_2)_{h-1}}[f(\tau_1,\tau_2)] - \mathbb{E}_{\tau_1 \sim (\hat{\mathbb{P}}_{k}, \pi_1)_{h},\tau_2 \sim (\hat{\mathbb{P}}_{k}, \pi_2)_h}[f(\tau_1,\tau_2)]
\end{align}

For a given $h$, suppose $d_h({\mathbb{P}}, \pi)$ denotes the state-action distribution in step $h$ when the agent interacts with the MDP with transition $\mathbb{P}$ using policy $\pi$, we have 
\begin{align}
    &\mathbb{E}_{\tau_1 \sim (\hat{\mathbb{P}}_k, \pi_1)_{h-1},\tau_2 \sim (\hat{\mathbb{P}}_k, \pi_2)_{h-1}}[f(\tau_1,\tau_2)] - \mathbb{E}_{\tau_1 \sim (\hat{\mathbb{P}}_{k}, \pi_1)_{h},\tau_2 \sim (\hat{\mathbb{P}}_{k}, \pi_2)_h}[f(\tau_1,\tau_2)] \\
    \leq & \mathbb{E}_{(s_{h,1},a_{h,1}) \sim d_h({\hat{\mathbb{P}}}_k, \pi_1), (s_{h,2},a_{h,2}) \sim d_h({\hat{\mathbb{P}}}_k, \pi_2)} \left[ \max_{V} \left(\left(\mathbb{P} - \hat{\mathbb{P}}_k\right)V(s_{h,1},a_{h,1})\right) +  \max_{V} \left(\left(\mathbb{P} - \hat{\mathbb{P}}_k\right)V(s_{h,2},a_{h,2})\right) \right] \\
    \leq & \mathbb{E}_{(s_{h,1},a_{h,1}) \sim d_h({\hat{\mathbb{P}}}_k, \pi_1), (s_{h,2},a_{h,2}) \sim d_h({\hat{\mathbb{P}}}_k, \pi_2)} \left[ b_{\mathbb{P},k}(s_{h,1}, a_{h,1}) +  b_{\mathbb{P},k}(s_{h,2}, a_{h,2})\right],
\end{align}
where the last inequality is due to $\mathbb{P} \in \caB_{\mathbb{P}, k}$ under event $\caE$ by Lemma~\ref{lemma: high confidence bound for hatP}. Summing over all $h \in [H]$, we can prove that 
\begin{align}
     \mathbb{E}_{\tau_1 \sim (\mathbb{P}, \pi_1),\tau_2 \sim (\mathbb{P}, \pi_2)}[f(\tau_1,\tau_2)] - \mathbb{E}_{\tau_1 \sim (\hat{\mathbb{P}}_{k}, \pi_1),\tau_2 \sim (\hat{\mathbb{P}}_{k}, \pi_2)}[f(\tau_1,\tau_2)] \leq \mathbb{E}_{\tau_1 \sim (\hat{\mathbb{P}}_{k}, \pi_1),\tau_2 \sim (\hat{\mathbb{P}}_{k}, \pi_2)}[b_{\mathbb{P},k}(\tau_1) + b_{\mathbb{P},k}(\tau_2)].
\end{align}
\end{proof}

\begin{lemma}
\label{lemma: pi* in caS}
    Under event $\mathcal{E}$, we have $\pi^* \in \caS_k$.
\end{lemma}

\begin{proof}
By Assumption~\ref{assumption: definition of pistar}, we know that 
\begin{align}
\label{inq: pi star definition}
   \mathbb{E}_{\tau_{0} \sim (\mathbb{P},\pi_{0}), \tau^* \sim  (\mathbb{P},\pi^*)} \mathbb{T}\left(\tau^{*}, \tau_{0}\right) \geq \frac{1}{2}, \forall \pi_{0}. 
\end{align}
We decompose the LHS of the above inequality into the following three terms:
\begin{align}
    \mathbb{E}_{\tau_{0} \sim (\mathbb{P},\pi_{0}), \tau^* \sim  ({\mathbb{P}},\pi^*)} \mathbb{T}\left(\tau^{*}, \tau_{0}\right) 
    =&  \mathbb{E}_{\tau_{0} \sim (\mathbb{P},\pi_{0}), \tau^* \sim  (\mathbb{P},\pi^*)} {\mathbb{T}}\left(\tau^{*}, \tau_{0}\right) - \mathbb{E}_{\tau_{0} \sim (\hat{\mathbb{P}}_k,\pi_{0}), \tau^* \sim  (\hat{\mathbb{P}}_k,\pi^*)} {\mathbb{T}}\left(\tau^{*}, \tau_{0}\right) \\
    & + \mathbb{E}_{\tau_{0} \sim (\hat{\mathbb{P}}_k,\pi_{0}), \tau^* \sim  (\hat{\mathbb{P}}_k,\pi^*)} \mathbb{T}\left(\tau^{*}, \tau_{0}\right) -  \mathbb{E}_{\tau_{0} \sim (\hat{\mathbb{P}}_k,\pi_{0}), \tau^* \sim  (\hat{\mathbb{P}}_k,\pi^*)} \hat{\mathbb{T}}_k\left(\tau^{*}, \tau_{0}\right)\\
    & + \mathbb{E}_{\tau_{0} \sim (\hat{\mathbb{P}}_k,\pi_{0}), \tau^* \sim  (\hat{\mathbb{P}}_k,\pi^*)} \hat{\mathbb{T}}_k\left(\tau^{*}, \tau_{0}\right).
\end{align}
By Lemma~\ref{lemma: error due to transition estimation}, we can upper bound the first term in the following way:
\begin{align}
    \mathbb{E}_{\tau_{0} \sim (\mathbb{P},\pi_{0}), \tau^* \sim  (\mathbb{P},\pi^*)} {\mathbb{T}}\left(\tau^{*}, \tau_{0}\right) - \mathbb{E}_{\tau_{0} \sim (\hat{\mathbb{P}}_k,\pi_{0}), \tau^* \sim  (\hat{\mathbb{P}}_k,\pi^*)} {\mathbb{T}}\left(\tau^{*}, \tau_{0}\right) \leq & \mathbb{E}_{\tau_1 \sim (\hat{\mathbb{P}}_{k}, \pi_1),\tau_2 \sim (\hat{\mathbb{P}}_{k}, \pi_2)}[b_{\mathbb{P},k}(\tau_1) + b_{\mathbb{P},k}(\tau_2)] .
\end{align}

By Lemma~\ref{lemma: high confidence bound for hatT}, we know that $\mathbb{T}(\tau_1,\tau_2) \in \caB_{\mathbb{T}, k}$ under event $\caE$. Therefore, 
\begin{align}
    &\mathbb{E}_{\tau_{0} \sim (\hat{\mathbb{P}}_k,\pi_{0}), \tau^* \sim  (\hat{\mathbb{P}}_k,\pi^*)} \mathbb{T}\left(\tau^{*}, \tau_{0}\right) -  \mathbb{E}_{\tau_{0} \sim (\hat{\mathbb{P}}_k,\pi_{0}), \tau^* \sim  (\hat{\mathbb{P}}_k,\pi^*)} \hat{\mathbb{T}}_k\left(\tau^{*}, \tau_{0}\right) \\
    \leq & \mathbb{E}_{\tau_{0} \sim (\hat{\mathbb{P}}_k,\pi_{0}), \tau^* \sim  (\hat{\mathbb{P}}_k,\pi^*)} \max _{f_{1}, f_{2} \in \mathcal{B}_{T, k}}\left|f_{1}\left(\tau_{1}, \tau_{2}\right)-f_{2}\left(\tau^*, \tau_{0}\right)\right| \\
    = & \mathbb{E}_{\tau_{0} \sim (\hat{\mathbb{P}}_k,\pi_{0}), \tau^* \sim  (\hat{\mathbb{P}}_k,\pi^*)} b_{\mathbb{T},k}(\tau^*, \tau_{0}).
\end{align}
Together with Inq.~\eqref{inq: pi star definition}, we have 
\begin{align}
    \mathbb{E}_{\tau_{0} \sim (\hat{\mathbb{P}}_k,\pi_{0}), \tau^* \sim  (\hat{\mathbb{P}}_k,\pi^*)} \left(\hat{\mathbb{T}}_k\left(\tau^{*}, \tau_{0}\right) + b_{\mathbb{T},k}(\tau^*, \tau_{0}) +  b_{\mathbb{P},k}(\tau^*) + b_{\mathbb{P},k}(\tau_0)\right) \geq \frac{1}{2}, \forall \pi_{0},
\end{align}
which indicates that $\pi^* \in \caS_k$.
\end{proof}

\begin{lemma}
\label{lemma: summation of preference bonus}
Under event $\caE$, it holds that 
\begin{align}
    \sum_{k=1}^{K} b_{\mathbb{T},k} (\tau_{k,1}, \tau_{k,2}) \leq  O(\sqrt{dK \log(K \mathcal{N}\left(\mathcal{F}_{\mathbb{T}}, 1 / K,\|\cdot\|_{\infty}\right) / \delta)}).
\end{align}
\end{lemma}

\begin{proof}
This lemma follows from the direct application of Lemma~\ref{lemma: auxiliary lemma for regret}. Note that under event $\caE$, we have $\max_{1\leq k \leq K} \operatorname{diam}(\caB_{\mathbb{T},k}|_{(\tau_1,\tau_2)_{1:k}}) \leq 2\sqrt{\beta_{\mathbb{T}}}$ by Lemma~\ref{lemma: high confidence bound for hatT}, and $f(\tau_1,\tau_2) \in [0,1], \forall f \in \caF_{\mathbb{T}}$. Therefore, 
\begin{align}
    \sum_{k=1}^{K} b_{\mathbb{T},k} (\tau_{k,1}, \tau_{k,2}) = & \sum_{k=1}^{K} \operatorname{diam} \left(\caB_{\mathbb{T}, k}|_{(\tau_{k,1}, \tau_{k,2})} \right) \leq O(\sqrt{dK \log(K \mathcal{N}\left(\mathcal{F}_{\mathbb{T}}, 1 / K,\|\cdot\|_{\infty}\right) / \delta)}).
\end{align}
\end{proof}

\begin{lemma}
\label{lemma: summation of transition bonus}
Under event $\caE$, 
\begin{align}
    \sum_{k=1}^{K} \left(b_{\mathbb{P},k} (\tau_{k,1}) + b_{\mathbb{P},k} (\tau_{k,2})\right) \leq  O(\sqrt{d_{\mathbb{P}}HK \log(K \mathcal{N}\left(\mathcal{F}_{\mathbb{T}}, 1 / K,\|\cdot\|_{\infty}\right) / \delta)}).
\end{align}
\end{lemma}

\begin{proof}
This lemma follows from the direct application of Lemma~\ref{lemma: auxiliary lemma for regret}. Note that under event $\caE$, we have $\sum_{i=1}^2\sum_{t=1}^{k} \sum_{h=1}^{H}\left(\left\langle\mathbb{P}\left(\cdot \mid s_{t, h,i}, a_{t, h,i}\right)-\hat{\mathbb{P}}_{k}\left(\cdot \mid s_{t, h,i}, a_{t, h,i}\right), V_{k,h,i}\right\rangle\right)^{2} \leq \beta_{\mathbb{P}}$ by Lemma~\ref{lemma: high confidence bound for hatP}, and $f(s,a, V) \in [0,1], \forall f \in \caF_{\mathbb{P}}$.  Therefore, 
\begin{align}
    \sum_{k=1}^{K} b_{\mathbb{P},k} (\tau_{k,1})+ b_{\mathbb{P},k} (\tau_{k,2}) = & \sum_{i=1}^2 \sum_{k=1}^{K} \sum_{h=1}^{H} b_{\mathbb{P},k} (s_{k,h,i},a_{k,h,i}) \leq O(\sqrt{d_{\mathbb{P}}HK \log(K \mathcal{N}\left(\mathcal{F}_{\mathbb{T}}, 1 / K,\|\cdot\|_{\infty}\right) / \delta)}).
\end{align}
\end{proof}

\begin{proof}[Proof of Theorem~\ref{theorem: dueling RL}] 
By the definition of regret, we have
\begin{align}
\label{inq: regret decompostion}
    Reg(K)  =& \sum_{k=1}^{K} \left(\mathbb{T}(\pi^*,\pi_{k,1}) + \mathbb{T}(\pi^*,\pi_{k,2}) -1\right) \\
    = & \sum_{k=1}^{K} \left(\mathbb{E}_{\tau^*\sim (\hat{\mathbb{P}}_k,\pi^*),\tau_1\sim (\hat{\mathbb{P}}_k,\pi_{k,1})}\hat{\mathbb{T}}_k(\tau^*,\tau_{1}) + \mathbb{E}_{\tau^*\sim (\hat{\mathbb{P}}_k,\pi^*),\tau_2\sim (\hat{\mathbb{P}}_k,\pi_{k,2})}\hat{\mathbb{T}}_k(\tau^*,\tau_{2}) -1 \right)\\
    &+ \sum_{k=1}^{K}  \left(\mathbb{E}_{\tau^*\sim ({\mathbb{P}},\pi^*),\tau_1\sim ({\mathbb{P}},\pi_{k,1})}{\mathbb{T}}(\tau^*,\tau_{1})  - \mathbb{E}_{\tau^*\sim (\hat{\mathbb{P}}_k,\pi^*),\tau_1\sim (\hat{\mathbb{P}}_k,\pi_{k,1})}{\mathbb{T}}(\tau^*,\tau_{1}) \right) \\
    &+ \sum_{k=1}^{K}  \left(\mathbb{E}_{\tau^*\sim ({\mathbb{P}},\pi^*),\tau_2\sim ({\mathbb{P}},\pi_{k,2})}{\mathbb{T}}(\tau^*,\tau_{2})  - \mathbb{E}_{\tau^*\sim (\hat{\mathbb{P}}_k,\pi^*),\tau_2\sim (\hat{\mathbb{P}}_k,\pi_{k,2})}{\mathbb{T}}(\tau^*,\tau_{2}) \right)\\
    &+ \sum_{k=1}^{K}   \mathbb{E}_{\tau^*\sim (\hat{\mathbb{P}}_k,\pi^*),\tau_1\sim (\hat{\mathbb{P}}_k,\pi_{k,1})}\left({\mathbb{T}}(\tau^*,\tau_{1}) - \hat{\mathbb{T}}_k(\tau^*,\tau_{1}) \right)\\
    &+ \sum_{k=1}^{K}   \mathbb{E}_{\tau^*\sim (\hat{\mathbb{P}}_k,\pi^*),\tau_2\sim (\hat{\mathbb{P}}_k,\pi_{k,2})}\left({\mathbb{T}}(\tau^*,\tau_{2}) - \hat{\mathbb{T}}_k(\tau^*,\tau_{2}) \right).
\end{align}
Since $\pi_{k,1} \in \caS_k$, we have 
\begin{align}
    \mathbb{E}_{\tau^*\sim (\hat{\mathbb{P}}_k,\pi^*),\tau_1\sim (\hat{\mathbb{P}}_k,\pi_{k,1})}\hat{\mathbb{T}}_k(\tau^*,\tau_{1}) - \frac{1}{2} \leq &\mathbb{E}_{\tau_1 \sim (\hat{\mathbb{P}}_k,\pi_{k,1}), \tau^* \sim (\hat{\mathbb{P}}_k,\pi^*)}\left( b_{\mathbb{T},k}(\tau_1,\tau^*) + b_{\mathbb{P},k}(\tau_1) + b_{\mathbb{P},k}(\tau^*) \right),
\end{align}
where the inequality is due to the definition of $\caS_k$.

By Lemma~\ref{lemma: error due to transition estimation}, we know that
\begin{align}
    \mathbb{E}_{\tau^*\sim ({\mathbb{P}},\pi^*),\tau_1\sim ({\mathbb{P}},\pi_{k,1})}{\mathbb{T}}(\tau^*,\tau_{1})  - \mathbb{E}_{\tau^*\sim (\hat{\mathbb{P}}_k,\pi^*),\tau_1\sim (\hat{\mathbb{P}}_k,\pi_{k,1})}{\mathbb{T}}(\tau^*,\tau_{1})  \leq \mathbb{E}_{\tau^* \sim (\hat{\mathbb{P}}_{k}, \pi^*),\tau_1 \sim (\hat{\mathbb{P}}_{k}, \pi_{k,1})}[b_{\mathbb{P},k}(\tau^*) + b_{\mathbb{P},k}(\tau_{1})].
\end{align}
From the definition of $b_{\mathbb{T},k}(\tau_1,\tau_2)$, we also know that
\begin{align}
    \mathbb{E}_{\tau^*\sim (\hat{\mathbb{P}}_k,\pi^*),\tau_1\sim (\hat{\mathbb{P}}_k,\pi_{k,1})}\left({\mathbb{T}}(\tau^*,\tau_{1}) - \hat{\mathbb{T}}_k(\tau^*,\tau_{1}) \right) \leq \mathbb{E}_{\tau^*\sim (\hat{\mathbb{P}}_k,\pi^*),\tau_2\sim (\hat{\mathbb{P}}_k,\pi_{k,2})} b_{\mathbb{T},k}(\tau^*,\tau_1)
\end{align}
Similarly, for the policy $\pi_{k,2}$, we also have
\begin{align}
    &\mathbb{E}_{\tau^*\sim (\hat{\mathbb{P}}_k,\pi^*),\tau_2\sim (\hat{\mathbb{P}}_k,\pi_{k,2})}\hat{\mathbb{T}}_k(\tau^*,\tau_{1}) - \frac{1}{2} \leq \mathbb{E}_{\tau_2 \sim (\hat{\mathbb{P}}_k,\pi_{k,2}), \tau^* \sim (\hat{\mathbb{P}}_k,\pi^*)}\left( b_{\mathbb{T},k}(\tau_2,\tau^*) + b_{\mathbb{P},k}(\tau_2) + b_{\mathbb{P},k}(\tau^*) \right) \\
    &\mathbb{E}_{\tau^*\sim ({\mathbb{P}},\pi^*),\tau_2\sim ({\mathbb{P}},\pi_{k,2})}{\mathbb{T}}(\tau^*,\tau_{2})  - \mathbb{E}_{\tau^*\sim (\hat{\mathbb{P}}_k,\pi^*),\tau_2\sim (\hat{\mathbb{P}}_k,\pi_{k,2})}{\mathbb{T}}(\tau^*,\tau_{2})  \leq  \mathbb{E}_{\tau^* \sim (\hat{\mathbb{P}}_{k}, \pi^*),\tau_2 \sim (\hat{\mathbb{P}}_{k}, \pi_{k,2})}[b_{\mathbb{P},k}(\tau^*) + b_{\mathbb{P},k}(\tau_{2})] \\
    &\mathbb{E}_{\tau^*\sim (\hat{\mathbb{P}}_k,\pi^*),\tau_2\sim (\hat{\mathbb{P}}_k,\pi_{k,2})}\left({\mathbb{T}}(\tau^*,\tau_{2}) - \hat{\mathbb{T}}_k(\tau^*,\tau_{2}) \right) \leq \mathbb{E}_{\tau^*\sim (\hat{\mathbb{P}}_k,\pi^*),\tau_2\sim (\hat{\mathbb{P}}_k,\pi_{k,2})} b_{\mathbb{T},k}(\tau^*,\tau_2).
\end{align}

Plugging the above inequalities back to Inq.~\eqref{inq: regret decompostion}, we have 
\begin{align}
    Reg(K) \leq & \sum_{k=1}^{K} \mathbb{E}_{\tau_1 \sim (\hat{\mathbb{P}}_k,\pi_{k,1}), \tau^* \sim (\hat{\mathbb{P}}_k, \pi^*)} \left(b_{\mathbb{T},k}(\tau_1,\tau^*) + b_{\mathbb{P},k}(\tau_1) + b_{\mathbb{P},k}(\tau^*)\right) \\
    & + \sum_{k=1}^{K} \mathbb{E}_{\tau_1 \sim (\hat{\mathbb{P}}_k,\pi_{k,1}), \tau^* \sim (\hat{\mathbb{P}}_k, \pi^*)} \left(b_{\mathbb{T},k}(\tau^*,\tau_1) + b_{\mathbb{P},k}(\tau_1) + b_{\mathbb{P},k}(\tau^*)\right) \\
    & + \sum_{k=1}^{K} \mathbb{E}_{\tau_2 \sim (\hat{\mathbb{P}}_k,\pi_{k,2}), \tau^* \sim (\hat{\mathbb{P}}_k, \pi^*)} \left(b_{\mathbb{T},k}(\tau_2,\tau^*) + b_{\mathbb{P},k}(\tau_2) + b_{\mathbb{P},k}(\tau^*)\right) \\
    & + \sum_{k=1}^{K} \mathbb{E}_{\tau_2 \sim (\hat{\mathbb{P}}_k,\pi_{k,2}), \tau^* \sim (\hat{\mathbb{P}}_k, \pi^*)} \left(b_{\mathbb{T},k}(\tau^*,\tau_2) + b_{\mathbb{P},k}(\tau_2) + b_{\mathbb{P},k}(\tau^*)\right) \\
    \leq & \sum_{k=1}^{K}2\mathbb{E}_{\tau_1 \sim (\hat{\mathbb{P}}_k,\pi_{k,1}), \tau_2 \sim (\hat{\mathbb{P}}_k,\pi_{k,2})}\left( b_{\mathbb{T},k}(\tau_1,\tau_2) + b_{\mathbb{P},k}(\tau_1) + b_{\mathbb{P},k}(\tau_2) \right),
\end{align}
where the second inequality follows the fact that $\pi_{k,1}$ and $\pi_{k,2}$ are the maximizer of $$\mathbb{E}_{ \tau_1 \sim (\hat{\mathbb{P}}_k,\pi_1), \tau_2 \sim (\hat{\mathbb{P}}_k,\pi_2)}\left(b_{\mathbb{T},k}(\tau_1,\tau_2) + b_{\mathbb{P},k}(\tau_1) + b_{\mathbb{P},k}(\tau_2)\right).$$

By definition, we have $0 \leq b_{\mathbb{T}, k}(\tau_1,\tau_2) \leq 1$ and $0 \leq b_{\mathbb{P},k}(\tau) \leq 1$. By Azuma's inequality, the following inequality holds with probability at least $1-\delta/2$,
\begin{align}
    &\sum_{k=1}^{K}2\mathbb{E}_{\tau_1 \sim (\hat{\PP}_k,\pi_{k,1}), \tau_2 \sim (\hat{\PP}_k,\pi_{k,2})}\left( b_{\mathbb{T},k}(\tau_1,\tau_2) + b_{\mathbb{P},k}(\tau_1) + b_{\mathbb{P},k}(\tau_2) \right) \\
    \leq & \sum_{k=1}^{K}2\left( b_{\mathbb{T},k}(\tau_{k,1},\tau_{k,2}) + b_{\mathbb{P},k}(\tau_{k,1}) + b_{\mathbb{P},k}(\tau_{k,2}) \right) + 4\sqrt{K \log(4/\delta)}.
\end{align}

By Lemma~\ref{lemma: summation of preference bonus} and Lemma~\ref{lemma: summation of transition bonus}, we can finally upper bound the total regret:
\begin{align}
    Reg(K) \leq O\left(\sqrt{d_{\mathbb{P}}HK \log(\mathcal{N}\left(\mathcal{F}_{\mathbb{P}}, 1/K,\|\cdot\|_{\infty}\right) / \delta)} + \sqrt{d_{\mathbb{T}}K \log(\mathcal{N}\left(\mathcal{F}_{\mathbb{T}}, 1/K,\|\cdot\|_{\infty}\right) / \delta)}\right).
\end{align}
\end{proof}

\section{RL with Once-per-episode Feedback}
\label{appendix: once-per-episode feedback}

\subsection{Algorithm}

We estimate $g^*$ by solving the following least-squares regression problem:
\begin{align} \label{eqn:estimate:g}
    \hat{g}_k = \argmin_{g \in \caF_{\mathbb{G}}} \sum_{t=1}^{k-1}[g(\tau_t) - y_t]^2 .
\end{align}
Then we construct the high-probability set for $g^*$:
\begin{align}
\label{eqn: confidence set for G}
    \caB_{\mathbb{G},k} = \left\{ g \mid  \sum_{t=1}^{k-1}\left(\hat{g}_{k}(\tau_t) - y_t\right)^2 \leq \beta_{\mathbb{G}}\right\}.
\end{align}

The transition estimation $\hat{\mathbb{P}}_{k}$ is the minimizer of the least-square loss:
\begin{align}
    \label{eqn: definition of hatP, trajectory feedback}
    \hat{\mathbb{P}}_{k}=\operatorname{argmin}_{\mathbb{P}^{\prime} \in \mathcal{P}} \sum_{t=1}^{k-1} \sum_{h=1}^{H}\left(\left\langle \mathbb{P}^{\prime}\left(\cdot \mid s_{t,h}, a_{t,h}\right), V_{k,h}\right\rangle-V_{k,h}(s_{k,h+1})\right)^{2},
\end{align}
where the value target $V_{k,h}$ defined in Line 8 of Algorithm~\ref{alg: RL with Trajectory Feedback} is the value function that maximizes the uncertainty in state-action pair $(s_{k,h},a_{k,h})$.

We define $L_{k}(\mathbb{P}_1,\mathbb{P}_2)$ as
\begin{align}
    L_{k}\left(\mathbb{P}_1, {\mathbb{P}}_{2}\right)=\sum_{t=1}^{k-1} \sum_{h=1}^{H}\left(\left\langle \mathbb{P}_1\left(\cdot \mid s_{t,h}, a_{t,h}\right)-{\mathbb{P}}_{2}\left(\cdot \mid s_{t,h}, a_{t,h}\right), V_{t,h}\right\rangle\right)^{2}.
\end{align}
We construct the high confidence set for transition $\mathbb{P}$, which is defined as
\begin{align}
\label{eqn: confidence set for P, trajectory feedback}
    \caB_{\mathbb{P},k} = \left\{ \tilde{\mathbb{P}} \mid L_k(\tilde{\mathbb{P}}, \hat{\mathbb{P}}_k) \leq \beta_{\mathbb{P}}\right\}.
\end{align}

Similar with Algorithm~\ref{alg: PbRL}, we calculate the policy set $\caS_{k}$, which contains near-optimal policies with minor sub-optimality gap. Finally, we execute the most exploratory policy in $\caS_k$.

\begin{algorithm}
  \caption{RL with Trajectory Feedback}
  \label{alg: RL with Trajectory Feedback}
    \begin{algorithmic}[1]
        \STATE Set $\beta_{\mathbb{G}} = \beta_{\mathbb{P}} =  8 \log(2K \mathcal{N}\left(\mathcal{F}_{\mathbb{T}}, 1/K,\|\cdot\|_{\infty}\right)/ \delta)$
        \FOR{episode $k = 1,\cdots, K$}
              \STATE Calculate the estimation $g_k$ using least-squares regression (Eqn.~\eqref{eqn:estimate:g}) 
              \STATE Construct the high-confidence set $\caB_{\mathbb{G},k}$ for the feedback function $g^{*}$ (Eqn.~\eqref{eqn: confidence set for G})
              \STATE Calculate the estimation $\hat{\mathbb{P}}_k$ using least-square regression (Eqn.~\ref{eqn: definition of hatP, trajectory feedback}).
              \STATE Construct the high-confidence set $\caB_{\mathbb{P},k}$ for transition $\mathbb{P}$ (Eqn.~\ref{eqn: confidence set for P, trajectory feedback})
              \STATE Define the bonus term $b_{\mathbb{P},k}(s, a) = \max_{\mathbb{P}_1,\mathbb{P}_2 \in \caB_{\mathbb{P},k}} \max_{V \in \caV} (\mathbb{P}_1-\mathbb{P}_2)V(s,a)$, and $b_{\mathbb{P},k}(\tau) = \sum_{(s,a) \in \tau} b_{\mathbb{P},k}(s, a)$.
              \STATE Define $V_{max,k,s,a} = \argmax_{V \in \caV} \max_{\mathbb{P}_1,\mathbb{P}_2 \in \caB_{\mathbb{P},k}} (\mathbb{P}_1-\mathbb{P}_2)V(s,a)$
              \STATE Define the bonus term $b_{\mathbb{G},k}(\tau) = \max_{g_1, g_2 \in \caB_{\mathbb{G},k}}\left|g_1(\tau) - g_2(\tau)\right|$
              \STATE Set $\caS_k = \left\{\pi \mid \mathbb{E}_{\tau \sim (\hat{\mathbb{P}}_k,\pi), \tau_0 \sim (\hat{\mathbb{P}}_k,\pi_0)}\left(\hat{g}_k(\tau)- \hat{g}_k(\tau_0) + b_{\mathbb{G},k}(\tau) + b_{\mathbb{G},k}(\tau_0) + b_{\mathbb{P},k}(\tau) + b_{\mathbb{P},k}(\tau_0) \right) \geq 0, \forall \pi_0 \in \Pi \right\}$
              \STATE Compute policy $\pi_{k} = \argmax_{\pi \in \caS_k} \mathbb{E}_{ \tau \sim (\hat{\mathbb{P}}_k,\pi)}\left(b_{\mathbb{G},k}(\tau) + b_{\mathbb{P},k}(\tau)\right) $
              \STATE Execute the policy $\pi_{k}$ for one episode, then observe the trajectory $\tau_{k}$ and the feedback $y_k$
      \ENDFOR
    \end{algorithmic}
  \end{algorithm}

\subsection{Theoretical Results}

\begin{theorem} [Restatement of Theorem~\ref{theorem: trajectory feedback, main text}]
\label{theorem: trajectory feedback}
With probability at least $1-\delta$, the regret of Algorithm~\ref{alg: RL with Trajectory Feedback} is upper bounded by $$Reg(K) \leq \tilde{O}( \sqrt{d_{\mathbb{P}}HK \log(\caN\left(\caF_{\mathbb{P}}, 1/K, \|\cdot\|_{\infty}\right)/\delta)} + \sqrt{d_{\mathbb{G}}K \log(\caN\left(\caF_{\mathbb{G}}, 1/K, \|\cdot\|_{\infty}\right)/\delta)}).$$
\end{theorem}

\begin{proof}
The proof shares almost the same idea with the analysis for PbRL. Therefore, we only explain the differences. Similar to Lemmas~\ref{lemma: high confidence bound for hatT} and \ref{lemma: high confidence bound for hatP}, we can show the following events happen with probability at least $1-\delta$,
\begin{align}
    \sum_{t=1}^{k-1}\left(\hat{g}_{k}-g^*\right)^{2}\left(\tau_{t}\right) \leq \beta_{\mathbb{T}}, \quad \sum_{t=1}^{k-1} \sum_{h=1}^{H}\left(\left\langle\mathbb{P}\left(\cdot \mid s_{t, h}, a_{t, h}\right)-\hat{\mathbb{P}}_{k}\left(\cdot \mid s_{t, h}, a_{t, h}\right), V_{t, h}\right\rangle\right)^{2} \leq  \beta_{\mathbb{P}}.
\end{align}

Denote the above event as $\caE$. Under event $\caE$, we also know that $\pi^* \in \caS_k$. Therefore, we upper bound the regret in the following way:
\begin{align}
\label{inq: regret decompostion, trajectory feedback}
    Reg(K)  =& \sum_{k=1}^{K}V^*(s_1) - V^{\pi_k}(s_1) \\
    = & \sum_{k=1}^{K}\mathbb{E}_{\tau^*\sim (\hat{\mathbb{P}}_k,\pi^*),\tau\sim (\hat{\mathbb{P}}_k,\pi_{k})} \left(\hat{g}_k(\tau^*)- \hat{g}_k(\tau)\right)  \\
    &+ \sum_{k=1}^{K}\mathbb{E}_{\tau^*\sim ({\PP},\pi^*),\tau\sim ({\PP},\pi_{k})}\left({g}^*(\tau^*)- {g}^*(\tau)\right)  - \mathbb{E}_{\tau^*\sim (\hat{\PP}_k,\pi^*),\tau\sim (\hat{\PP}_k,\pi_{k})}\left({g}^*(\tau^*)- {g}^*(\tau)\right) \\
    &+ \sum_{k=1}^{K}   \mathbb{E}_{\tau^*\sim (\hat{\PP}_k,\pi^*),\tau\sim (\hat{\PP}_k,\pi_{k})}\left(\left({g}^*(\tau^*)- {g}^*(\tau)\right) - \left(\hat{g}_k(\tau^*)- \hat{g}_k(\tau)\right) \right).
\end{align}
Since $\pi_{k} \in \caS_k$, we have 
\begin{align}
    \mathbb{E}_{\tau \sim (\hat{\mathbb{P}}_k,\pi_k), \tau^* \sim (\hat{\mathbb{P}}_k,\pi^*)}\left(\hat{g}_k(\tau)- \hat{g}_k(\tau^*) \right) \leq \mathbb{E}_{\tau \sim (\hat{\mathbb{P}}_k,\pi), \tau^* \sim (\hat{\mathbb{P}}_k,\pi^*)}\left( b_{\mathbb{G},k}(\tau) + b_{\mathbb{G},k}(\tau^*) + b_{\mathbb{P},k}(\tau) + b_{\mathbb{P},k}(\tau^*) \right) 
\end{align}

Similarly, by Lemma~\ref{lemma: error due to transition estimation}, we know that
\begin{align}
     &\sum_{k=1}^{K}\mathbb{E}_{\tau^*\sim ({\PP},\pi^*),\tau\sim ({\PP},\pi_{k})}\left({g}^*(\tau^*)- {g}^*(\tau)\right)  - \mathbb{E}_{\tau^*\sim (\hat{\PP}_k,\pi^*),\tau\sim (\hat{\PP}_k,\pi_{k})}\left({g}^*(\tau^*)- {g}^*(\tau)\right) \\
     \leq &\mathbb{E}_{\tau^* \sim (\hat{\mathbb{P}}_{k}, \pi^*),\tau \sim (\hat{\mathbb{P}}_{k}, \pi_{k})}[b_{\mathbb{P},k}(\tau^*) + b_{\mathbb{P},k}(\tau)].
\end{align}
From the definition of $b_{\mathbb{G},k}(\tau)$, we also know that
\begin{align}
    \mathbb{E}_{\tau\sim (\hat{\PP}_k,\pi_{k})}\left({g}^*(\tau) - \hat{g}(\tau) \right) \leq \mathbb{E}_{\tau\sim (\hat{\PP}_k,\pi_{k})} b_{\mathbb{G},k}(\tau), \\
    \mathbb{E}_{\tau^*\sim (\hat{\PP}_k,\pi^*)}\left({g}^*(\tau^*) - \hat{g}(\tau^*) \right) \leq \mathbb{E}_{\tau\sim (\hat{\PP}_k,\pi_{k})} b_{\mathbb{G},k}(\tau^*).
\end{align}

Plugging the above inequalities back to Inq.~\eqref{inq: regret decompostion, trajectory feedback}, we have 
\begin{align}
    Reg(K) \leq & \sum_{k=1}^{K} \mathbb{E}_{\tau \sim (\hat{\PP}_k,\pi_{k}), \tau^* \sim (\hat{\PP}_k, \pi^*)} \left(b_{\mathbb{G}, k}(\tau)+b_{\mathbb{G}, k}\left(\tau^{*}\right) + b_{\mathbb{\PP},k}(\tau) + b_{\mathbb{\PP},k}(\tau^*)\right) \\
    & + \sum_{k=1}^{K} \mathbb{E}_{\tau \sim (\hat{\PP}_k,\pi_{k}), \tau^* \sim (\hat{\PP}_k, \pi^*)} \left(b_{\mathbb{G}, k}(\tau)+b_{\mathbb{G}, k}\left(\tau^{*}\right) + b_{\mathbb{P},k}(\tau) + b_{\mathbb{P},k}(\tau^*)\right) \\
    \leq & \sum_{k=1}^{K}4\mathbb{E}_{\tau \sim (\hat{\PP}_k,\pi_{k})}\left( b_{\mathbb{G},k}(\tau) + b_{\mathbb{P},k}(\tau) \right),
\end{align}
where the second inequality follows the fact that $\pi_{k}$ is the maximizer of $\mathbb{E}_{ \tau \sim (\hat{\PP}_k,\pi)}\left( b_{\mathbb{G},k}(\tau) + b_{\mathbb{P},k}(\tau) \right).$ By definition, we have $0 \leq b_{\mathbb{G}, k}(\tau) \leq 1$ and $0 \leq b_{\mathbb{P},k}(\tau) \leq 1$. By Azuma's inequality, the following inequality holds with probability at least $1-\delta/2$,
\begin{align}
    &\sum_{k=1}^{K}4\mathbb{E}_{\tau \sim (\hat{\PP}_k,\pi_{k})}\left( b_{\mathbb{G},k}(\tau) + b_{\mathbb{P},k}(\tau) \right) \\
    \leq & \sum_{k=1}^{K}4\left( b_{\mathbb{G},k}(\tau_k) + b_{\mathbb{P},k}(\tau_k) \right) + 8\sqrt{K \log(4/\delta)}.
\end{align}

We upper bound the summation of bonus with the help of Lemma~\ref{lemma: auxiliary lemma for regret}. Finally, we have
\begin{align}
    Reg(K) \leq O\left(\sqrt{d_{\mathbb{P}}HK \log(\mathcal{N}\left(\mathcal{F}_{\mathbb{P}}, 1/K,\|\cdot\|_{\infty}\right) / \delta)} + \sqrt{d_{\mathbb{G}}K \log(\mathcal{N}\left(\mathcal{F}_{\mathbb{T}}, 1/K,\|\cdot\|_{\infty}\right) / \delta)}\right).
\end{align}
\end{proof}

\section{Proof of Theorem~\ref{theorem:lb:once:per:episode}}
\label{appendix: lower bound once-per-episode feedback}
\begin{proof}
    For any $\tau = (s_1, a_1, \cdots, s_H, a_H)$, let $g^*(\tau) = \sum_{h=1}^H r(s_h, a_h)$. Then the RL with once-per-episode feedback problem reduces to the traditional RL (RL with reward signals). Thus, the lower bound $\Omega(\tilde{d}_{\mathbb{P}}\sqrt{K})$\footnote{Their lower bound is $\Omega(\tilde{d}_{\mathbb{P}}\sqrt{H^3K})$ because they consider the setting that the total reward is bounded by $H$ and the transition kernel is time-inhomogeneous.} established in \citet{zhou2021nearly} immediately implies the same regret lower bound for the RL with once-per-episode feedback problem. Meanwhile, by regarding the trajectory as an ``arm'', the lower bound $\Omega(\tilde{d}_{\mathbb{T}} \sqrt{K})$ for linear bandits implies the lower bound $\Omega(\tilde{d}_{\mathbb{T}}\sqrt{K)}$ for our setting. Putting these two lower bound together, we obtain $Reg(K) \ge \Omega(\max\{\tilde{d}_{\mathbb{P}} \sqrt{K}, \tilde{d}_{\mathbb{T}} \sqrt{K}\})$, which equivalents to $Reg(K) \ge \Omega(\tilde{d}_{\mathbb{P}}\sqrt{K} + \tilde{d}_{\mathbb{T}}\sqrt{K})$. Therefore, we finish the proof.
\end{proof}

\section{Proof of Theorem~\ref{theorem: pairwise}} \label{appendix:pairwise:comparison}

We also need the following concentration lemmas to guarantee that the true preference $\mathbb{T}(\cdot, \cdot) \in \caB_{\mathbb{T},k}$ with high probability and the true transition kernel $\mathbb{P}(s'|s,a) \in \caB_{\mathbb{P},k}$ with high probability, respectively.

\begin{lemma}
    \label{lemma: high confidence bound for hatT 2}
    Fix $\delta \in (0,1)$, with probability at least $1-\delta$, for all $k \in [K]$,
    \begin{align}
        \sum_{t=1}^{k-1}\sum_{i = 1}^n \sum_{j = i + 1}^n \left(\hat{\mathbb{T}}_{k} - \mathbb{T}\right)^2(\tau_{t,i}, \tau_{t,j}) \leq  \beta_{\mathbb{T}}.
    \end{align}
    \end{lemma}
    
    \begin{proof}
    This lemma can be proved by the direct application of Lemma~\ref{lemma: auxiliary Lemma for confidence set}.
    \end{proof}
    
    
    \begin{lemma}
    \label{lemma: high confidence bound for hatP 2}
    Fix $\delta \in (0,1)$, with probability at least $1-\delta$, for all $k \in [K]$,
    \begin{align}
        L_{k}\left(\mathbb{P}, \hat{\mathbb{P}}_{k}\right)=\sum_{i=1}^n\sum_{t=1}^{k-1} \sum_{h=1}^{H}\left(\left\langle \mathbb{P}\left(\cdot \mid s_{t,h,i}, a_{t,h,i}\right)-\hat{\mathbb{P}}_{k}\left(\cdot \mid s_{t,h,i}, a_{t,h,i}\right), V_{t,h,i}\right\rangle\right)^{2} \leq \beta_{\mathbb{P}}.
    \end{align}
    \end{lemma}

    \begin{proof}
        This lemma can be proved by the direct application of Lemma~\ref{lemma: auxiliary Lemma for confidence set}.
    \end{proof}
        
        
    With slight abuse of notation, we denote the high-probability event in Lemmas~\ref{lemma: high confidence bound for hatT 2} and \ref{lemma: high confidence bound for hatP 2} as $\mathcal{E}$.

    \begin{lemma}
    \label{lemma: error due to transition estimation 2}
    Under event $\caE$, for any two policies $\pi_1, \pi_2$ and scalar function $f: \operatorname{Traj} \times  \operatorname{Traj} \rightarrow [0,1]$, we have
    \begin{align}
        \mathbb{E}_{\tau_1 \sim (\mathbb{P}, \pi_1),\tau_2 \sim (\mathbb{P}, \pi_2)}[f(\tau_1,\tau_2)] - \mathbb{E}_{\tau_1 \sim (\hat{\mathbb{P}}_{k}, \pi_1),\tau_2 \sim (\hat{\mathbb{P}}_{k}, \pi_2)}[f(\tau_1,\tau_2)] \leq \mathbb{E}_{\tau_1 \sim (\hat{\mathbb{P}}_{k}, \pi_1),\tau_2 \sim (\hat{\mathbb{P}}_{k}, \pi_2)}[b_{\mathbb{P},k}(\tau_1) + b_{\mathbb{P},k}(\tau_2)] .
    \end{align}
    \end{lemma}
    \begin{proof}
        The proof is the same as that of Lemma~\ref{lemma: error due to transition estimation} and we omit it here to avoid repetition.
    \end{proof}
    
    \begin{lemma} \label{lemma: pi* in caS 2}
    Under event $\mathcal{E}$, we have $\pi^* \in \caS_k$.
    \end{lemma}
    \begin{proof}
    The proof is the same as that of Lemma~\ref{lemma: pi* in caS} and we omit it here to avoid repetition.
    \end{proof}

 With these lemmas, we are ready to provide the proof of Theorem~\ref{theorem: pairwise}.

\begin{proof}[Proof of Theorem~\ref{theorem: pairwise}] 
    By the definition of regret, we have
    \begin{align}
    \label{inq: regret decompostion 2}
         Reg(K)  =& \sum_{k=1}^{K} \sum_{i = 1}^n\left(\mathbb{T}(\pi^*,\pi_{k,i}) - \frac{1}{2}\right) \\
        = & \sum_{k=1}^{K}\sum_{i = 1}^n \left(\mathbb{E}_{\tau^*\sim (\hat{\PP}_k,\pi^*),\tau_i\sim (\hat{\PP}_k,\pi_{k,i})}\hat{\mathbb{T}}(\tau^*,\tau_{i}) - \frac{1}{2} \right)\\
        &+ \sum_{k=1}^{K}  \sum_{i = 1}^n \left(\mathbb{E}_{\tau^*\sim ({\PP},\pi^*),\tau_i \sim ({\PP},\pi_{k,i})}{\mathbb{T}}(\tau^*,\tau_{i})  - \mathbb{E}_{\tau^*\sim (\hat{\PP}_k,\pi^*),\tau_i\sim (\hat{\PP}_k,\pi_{k,i})}{\mathbb{T}}(\tau^*,\tau_{i}) \right)\\
        &+ \sum_{k=1}^{K}  \sum_{i = 1}^n \mathbb{E}_{\tau^*\sim (\hat{\PP}_k,\pi^*),\tau_i \sim (\hat{\PP}_k,\pi_{k,i})}\left({\mathbb{T}}(\tau^*,\tau_{i}) - \hat{\mathbb{T}}(\tau^*,\tau_{i}) \right).
    \end{align}
    Note that $\pi_{k,i} \in \caS_k$ for all $i \in [n]$, we have 
    \begin{align}
        \mathbb{E}_{\tau^*\sim (\hat{\PP}_k,\pi^*),\tau_i\sim (\hat{\PP}_k,\pi_{k,i})}\hat{\mathbb{T}}(\tau^*,\tau_{i}) - \frac{1}{2} \leq &\mathbb{E}_{\tau_i \sim (\hat{\PP}_k,\pi_{k,1}), \tau^* \sim (\hat{\PP}_k,\pi^*)}\left( b_{\mathbb{T},k}(\tau_i,\tau^*) + b_{\mathbb{P},k}(\tau_i) + b_{\mathbb{P},k}(\tau^*) \right),
    \end{align}
    By Lemma~\ref{lemma: error due to transition estimation 2}, we have that
    \begin{align}
        \mathbb{E}_{\tau^*\sim ({\PP},\pi^*),\tau_i\sim ({\PP},\pi_{k,i})}{\mathbb{T}}(\tau^*,\tau_{i})  - \mathbb{E}_{\tau^*\sim (\hat{\PP}_k,\pi^*),\tau_i\sim (\hat{\PP}_k,\pi_{k,i})}{\mathbb{T}}(\tau^*,\tau_{i})  \leq \mathbb{E}_{\tau^* \sim (\hat{\mathbb{P}}_{k}, \pi^*),\tau_i \sim (\hat{\mathbb{P}}_{k}, \pi_{k,i})}[b_{\mathbb{P},k}(\tau^*) + b_{\mathbb{P},k}(\tau_{i})].
    \end{align}
    By the definition of $b_{\mathbb{T},k}$, we also know that
    \begin{align}
        \mathbb{E}_{\tau^*\sim (\hat{\PP}_k,\pi^*),\tau_i \sim (\hat{\PP}_k,\pi_{k,i})}\left({\mathbb{T}}(\tau^*,\tau_{i}) - \hat{\mathbb{T}}(\tau^*,\tau_{i}) \right) \leq b_{\mathbb{T},k}(\tau^*,\tau_i).
    \end{align}
    Plugging the above inequalities back to Inq.~\eqref{inq: regret decompostion 2}, we have 
    \begin{align}
         Reg(K) \leq & \sum_{k=1}^{K}\sum_{i = 1}^n \mathbb{E}_{\tau_i \sim (\hat{\PP}_k,\pi_{k,i}), \tau^* \sim (\hat{\PP}_k, \pi^*)} \left(b_{\mathbb{T},k}(\tau_i,\tau^*) + b_{\mathbb{P},k}(\tau_i) + b_{\mathbb{P},k}(\tau^*)\right) \\
        & + \sum_{k=1}^{K}\sum_{i = 1}^n \mathbb{E}_{\tau_i \sim (\hat{\PP}_k,\pi_{k,i}), \tau^* \sim (\hat{\PP}_k, \pi^*)} \left(b_{\mathbb{T},k}(\tau^*,\tau_i) + b_{\mathbb{P},k}(\tau_i) + b_{\mathbb{P},k}(\tau^*)\right) \\
        = & 2\sum_{k=1}^{K}\sum_{i = 1}^n \mathbb{E}_{\tau_i \sim (\hat{\PP}_k,\pi_{k,i}), \tau^* \sim (\hat{\PP}_k, \pi^*)} \left(b_{\mathbb{T},k}(\tau^*,\tau_i) + b_{\mathbb{P},k}(\tau_i) + b_{\mathbb{P},k}(\tau^*)\right),
    \end{align}
    For any $(k, i) \in [K] \times [n]$, we have 
    \begin{align}
        &\sum_{i = 1}^n \mathbb{E}_{\tau_i \sim (\hat{\PP}_k,\pi_{k,i}), \tau^* \sim (\hat{\PP}_k, \pi^*)} \left(b_{\mathbb{T},k}(\tau^*,\tau_i) + b_{\mathbb{P},k}(\tau_i) + b_{\mathbb{P},k}(\tau^*)\right) \\
        =& \frac{1}{n-1} \cdot \sum_{i = 1}^n \sum_{j \neq i} \mathbb{E}_{\tau_j \sim (\hat{\PP}_k,\pi_{k,j}), \tau^* \sim (\hat{\PP}_k, \pi^*)} \left(b_{\mathbb{T},k}(\tau^*,\tau_j) + b_{\mathbb{P},k}(\tau_j) + b_{\mathbb{P},k}(\tau^*)\right).
    \end{align}
    Note that
    \begin{align}
    (\pi_{k,1},\pi_{k,2}, \cdots, \pi_{k, n}) &= \argmax_{\pi_1,\pi_2, \cdots, \pi_n \in \caS_k} \sum_{i = 1}^n \sum_{j = i+1}^{n} \mathbb{E}_{ \tau_i \sim (\hat{\mathbb{P}}_k,\pi_i), \tau_j \sim (\hat{\mathbb{P}}_k,\pi_j)} \left(b_{\mathbb{T},k}(\tau_i,\tau_j) + b_{\mathbb{P},k}(\tau_i) + b_{\mathbb{P},k}(\tau_j)\right) \\
    & = \argmax_{\pi_1,\pi_2, \cdots, \pi_n \in \caS_k} \frac{1}{2} \sum_{i = 1}^n \sum_{j \neq i} \mathbb{E}_{ \tau_i \sim (\hat{\mathbb{P}}_k,\pi_i), \tau_j \sim (\hat{\mathbb{P}}_k,\pi_j)} \left(b_{\mathbb{T},k}(\tau_i,\tau_j) + b_{\mathbb{P},k}(\tau_i) + b_{\mathbb{P},k}(\tau_j)\right) \\
    & =  \argmax_{\pi_1,\pi_2, \cdots, \pi_n \in \caS_k} \sum_{i = 1}^n \sum_{j \neq i} \mathbb{E}_{ \tau_i \sim (\hat{\mathbb{P}}_k,\pi_i), \tau_j \sim (\hat{\mathbb{P}}_k,\pi_j)} \left(b_{\mathbb{T},k}(\tau_i,\tau_j) + b_{\mathbb{P},k}(\tau_i) + b_{\mathbb{P},k}(\tau_j)\right),
    \end{align}
    together with the fact that $\pi^* \in \caS_k$ (Lemma~\ref{lemma: pi* in caS 2}), we have for any $i \in [n]$,
    \begin{align}
        &\sum_{j \neq i} \mathbb{E}_{\tau_j \sim (\hat{\PP}_k,\pi_{k,j}), \tau^* \sim (\hat{\PP}_k, \pi^*)} \left(b_{\mathbb{T},k}(\tau^*,\tau_j) + b_{\mathbb{P},k}(\tau_j) + b_{\mathbb{P},k}(\tau^*)\right) \\
        &\qquad + \sum_{j, l \neq i} \mathbb{E}_{\tau_j \sim (\hat{\PP}_k,\pi_{k,j}), \tau_l \sim (\hat{\PP}_k, \pi_{k, l})} \left(b_{\mathbb{T},k}(\tau_j,\tau_l) + b_{\mathbb{P},k}(\tau_j) + b_{\mathbb{P},k}(\tau_l)\right) \\
        \le& \sum_{i = 1}^n \sum_{j \neq i} \mathbb{E}_{\tau_i \sim (\hat{\PP}_k,\pi_{k,i}), \tau_j \sim (\hat{\PP}_k, \pi_j)} \left(b_{\mathbb{T},k}(\tau_i,\tau_j) + b_{\mathbb{P},k}(\tau_i) + b_{\mathbb{P},k}(\tau_j)\right),
    \end{align}
    which equivalents to 
    \begin{align}
        &\sum_{j \neq i} \mathbb{E}_{\tau_j \sim (\hat{\PP}_k,\pi_{k,j}), \tau^* \sim (\hat{\PP}_k, \pi^*)} \left(b_{\mathbb{T},k}(\tau^*,\tau_j) + b_{\mathbb{P},k}(\tau_j) + b_{\mathbb{P},k}(\tau^*)\right) \\
        \le & \sum_{j \neq i} \mathbb{E}_{\tau_j \sim (\hat{\PP}_k,\pi_{k,j}), \tau_i \sim (\hat{\PP}_k, \pi_i)} \left(b_{\mathbb{T},k}(\tau_i,\tau_j) + b_{\mathbb{P},k}(\tau_i) + b_{\mathbb{P},k}(\tau_i)\right).
    \end{align}
    Taking summation over $i \in [n]$ gives that 
    \begin{align}
        &\sum_{i = 1}^n \sum_{j \neq i} \mathbb{E}_{\tau_i \sim (\hat{\PP}_k,\pi_{k,j}), \tau^* \sim (\hat{\PP}_k, \pi^*)} \left(b_{\mathbb{T},k}(\tau^*,\tau_j) + b_{\mathbb{P},k}(\tau_j) + b_{\mathbb{P},k}(\tau^*)\right) \\
        \le & \sum_{i = 1}^n \sum_{j \neq i} \mathbb{E}_{\tau_i \sim (\hat{\PP}_k,\pi_{k,j}), \tau_j \sim (\hat{\PP}_k, \pi_j)} \left(b_{\mathbb{T},k}(\tau_i,\tau_j) + b_{\mathbb{P},k}(\tau_j) + b_{\mathbb{P},k}(\tau_i)\right).
    \end{align}

    By definition, we have $0 \leq b_{\mathbb{T}, k}(\tau_1,\tau_2) \leq H$ and $0 \leq b_{\mathbb{P},k}(\tau) \leq 1$. By Azuma's inequality, the following inequality holds with probability at least $1-\delta/2$,
    \begin{align}
        &\sum_{k = 1}^K\sum_{i = 1}^n \sum_{j \neq i} \mathbb{E}_{\tau_i \sim (\hat{\PP}_k,\pi_{k,j}), \tau_j \sim (\hat{\PP}_k, \pi_j)} \left(b_{\mathbb{T},k}(\tau_i,\tau_j) + b_{\mathbb{P},k}(\tau_j) + b_{\mathbb{P},k}(\tau_i)\right) \\
        \le& \sum_{k = 1}^K\sum_{i = 1}^n \sum_{j \neq i}  \left(b_{\mathbb{T},k}(\tau_{k, i},\tau_{k, j}) + b_{\mathbb{P},k}(\tau_{k, j}) + b_{\mathbb{P},k}(\tau_{k, i})\right) + 2Hn\sqrt{K\log(4/\delta)} \\
        =& 2 \sum_{k = 1}^K\sum_{i = 1}^n \sum_{j = i+1}^n  \left(b_{\mathbb{T},k}(\tau_{k, i},\tau_{k, j}) + b_{\mathbb{P},k}(\tau_{k, j}) + b_{\mathbb{P},k}(\tau_{k, i})\right) + 2Hn\sqrt{K\log(4/\delta)}
    \end{align}

   \begin{lemma} \label{lemma: summation of preference bonus 2}
   Under event $\caE$, it holds that
   \begin{align}
     \sum_{k=1}^{K}\sum_{i = 1}^n \sum_{j = i + 1}^n b_{\mathbb{T},k} (\tau_{k,i}, \tau_{k,j}) \leq  O(\sqrt{dKn^2 \log(K \mathcal{N}\left(\mathcal{F}_{\mathbb{T}}, 1 / (Kn^2),\|\cdot\|_{\infty}\right) / \delta)}).
   \end{align}
   \end{lemma}

   \begin{proof}
   This lemma follows from the direct application of Lemma~\ref{lemma: auxiliary lemma for regret}. Note that under event $\caE$, we have $\max_{1\leq k \leq K} \operatorname{diam}(\caB_{\mathbb{T},k}|_{(\tau_1,\tau_2)_{1:k}}) \leq 2\sqrt{\beta_{\mathbb{T}}}$ by Lemma~\ref{lemma: high confidence bound for hatT 2}, and $f(\tau_1,\tau_2) \in [0,1], \forall f \in \caF_{\mathbb{T}}$. Therefore, 
   \begin{align}
     \sum_{k=1}^{K}\sum_{i = 1}^n \sum_{j = i + 1}^n b_{\mathbb{T},k} (\tau_{k,i}, \tau_{k,j}) = & \sum_{k=1}^{K}\sum_{i = 1}^n \sum_{j \neq i} \operatorname{diam} \left(\caB_{\mathbb{T}, k}|_{(\tau_{k,1}, \tau_{k,2})} \right) \leq O(\sqrt{dKn^2 \log(K \mathcal{N}\left(\mathcal{F}_{\mathbb{T}}, 1 / (Kn^2),\|\cdot\|_{\infty}\right) / \delta)}).
   \end{align}
   \end{proof}

   \begin{lemma} \label{lemma: summation of transition bonus 2}
   Under event $\caE$, it holds that
   \begin{align}
     \sum_{k=1}^{K}\sum_{i = 1}^n b_{\mathbb{P},k} (\tau_{k,i})  \leq  O(\sqrt{d_{\mathbb{P}}HKn \log(K \mathcal{N}\left(\mathcal{F}_{\mathbb{P}}, 1 / (Kn),\|\cdot\|_{\infty}\right) / \delta)}).
   \end{align}
   \end{lemma}

   \begin{proof}
    This lemma follows from the direct application of Lemma~\ref{lemma: auxiliary lemma for regret}. Note that under event $\caE$, we have $\sum_{i=1}^n\sum_{t=1}^{k} \sum_{h=1}^{H}\left(\left\langle\mathbb{P}\left(\cdot \mid s_{t, h,i}, a_{t, h,i}\right)-\hat{\mathbb{P}}_{k}\left(\cdot \mid s_{t, h,i}, a_{t, h,i}\right), V_{k,h,i}\right\rangle\right)^{2} \leq \beta_{\mathbb{P}}$ by Lemma~\ref{lemma: high confidence bound for hatP 2}, and $f(s,a, V) \in [0,1], \forall f \in \caF_{\mathbb{P}}$. Therefore, 
   \begin{align}
    \sum_{k=1}^{K}\sum_{i = 1}^n b_{\mathbb{P},k}(\tau_{k,i}) = & \sum_{i=1}^n \sum_{k=1}^{K} \sum_{h=1}^{H} b_{\mathbb{P},k} (s_{k,h,i},a_{k,h,i}) \leq O(\sqrt{d_{\mathbb{P}}HKn \log(K \mathcal{N}\left(\mathcal{F}_{\mathbb{P}}, 1 / (Kn),\|\cdot\|_{\infty}\right) / \delta)}).
   \end{align}
   \end{proof}

    By Lemmas~\ref{lemma: summation of preference bonus 2} and \ref{lemma: summation of transition bonus 2}, we can finally upper bound the total regret:
    \begin{align}
        Reg(K) \leq O\left(\sqrt{d_{\mathbb{P}}HnK \cdot  \log(\mathcal{N}\left(\mathcal{F}_{\mathbb{P}}, 1/(Kn),\|\cdot\|_{\infty}\right) / \delta)} + \sqrt{d_{\mathbb{T}}K \cdot  \log(\mathcal{N}\left(\mathcal{F}_{\mathbb{T}}, 1/(Kn^2),\|\cdot\|_{\infty}\right) / \delta)}\right).
    \end{align}
    \end{proof}

\section{Auxiliary Lemmas}
Let $\left(X_{p}, Y_{p}\right)_{p=1,2, \ldots}$ be a sequence of random elements, $X_p \in \caX$ for some measurable set $\caX$ and $Y_p \in \mathbb{R}$. Let $\caF$ be a subset of the set of real-valued measurable functions with domain $\caX$. Let $\mathbb{F} = (\mathbb{F}_p)_{p=0,1,\cdots}$ be a filtration such that for all $p \geq 1$, $(X_1,Y_1,\cdots, X_{p-1},Y_{p-1}, X_{p})$ is $\mathbb{F}_{p-1}$ measurable and such that there exists some function $f_{\star} \in \caF$ such that $\mathbb{E}\left[Y_{p} \mid \mathbb{F}_{p-1}\right]=f_{*}\left(X_{p}\right)$ holds for all $p \geq 1$. The (nonlinear) least square predictor given $(X_1,Y_1,\cdots,X_{t}, Y_t)$ is $\hat{f}_{t}=\operatorname{argmin}_{f \in \mathcal{F}} \sum_{p=1}^{t}\left(f\left(X_{p}\right)-Y_{p}\right)^{2}$. We say that $Z$ is conditionally $\rho$-subgaussion given the $\sigma$-algebra $\mathbb{F}$ is for all $\lambda \in \mathbb{R}$, $\log \mathbb{E}[\exp (\lambda Z) \mid \mathbb{F}] \leq \frac{1}{2} \lambda^{2} \rho^{2}$. For $\alpha > 0$, let $N_{\alpha}$ be the $\| \cdot\|_{\infty}$-covering number of $\caF$ at scale $\alpha$. For $\beta > 0$, define
\begin{align}
    \mathcal{F}_{t}(\beta)=\left\{f \in \mathcal{F}: \sum_{p=1}^{t}\left(f\left(X_{p}\right)-\hat{f}_{t}\left(X_{p}\right)\right)^{2} \leq \beta\right\}.
\end{align}
\begin{lemma}
\label{lemma: auxiliary Lemma for confidence set}
(Theorem 5 of \cite{ayoub2020model}). Let $\mathbb{F}$ be the filtration defined above and assume that the functions in $\caF$ are bounded by the positive constant $C > 0$. Assume that for each $s \geq 1$, $\left(Y_{p}-f_{*}\left(X_{p}\right)\right)$ is conditionally $\sigma$-subgaussian given $\mathbb{F}_{p-1}$. Then, for any $\alpha > 0$, with probability $1-\delta$, for all $t \geq 1$, $f_{*} \in \mathcal{F}_{t}\left(\beta_{t}(\delta, \alpha)\right)$, where 
\begin{align}
    \beta_{t}(\delta, \alpha)=8 \sigma^{2} \log \left(2 N_{\alpha} / \delta\right)+4 t \alpha\left(C+\sqrt{\sigma^{2} \log (4 t(t+1) / \delta)}\right).
\end{align}
\end{lemma}

\begin{lemma}
\label{lemma: auxiliary lemma for regret}
(Lemma 5 of \cite{russo2014learning}). Let $\caF \in B_{\infty} (\caX, C)$ be a set of functions bounded by $C > 0$, $(\caF_t)_{t \geq 1}$ and $(x_t)_{t \geq 1}$ be sequences such that $\mathcal{F}_{t} \subset \mathcal{F}$ and $x_t \in \caX$ hold for $t \geq 1$. Let $\left.\mathcal{F}\right|_{x_{1: t}}=\left\{\left(f\left(x_{1}\right), \ldots, f\left(x_{t}\right)\right): f \in \mathcal{F}\right\}\left(\subset \mathbb{R}^{t}\right)$ and for $S \subset \mathbb{R}^{t}$, let $\operatorname{diam}(S)=\sup _{u, v \in S}\|u-v\|_{2}$ be the diameter of $S$. Then, for any $T \geq 1$ and $\alpha > 0$ it holds that 
\begin{align}
    \sum_{t=1}^{T} \operatorname{diam}\left(\left.\mathcal{F}_{t}\right|_{x_{t}}\right) \leq \alpha+C(d \wedge T)+2 \delta_{T} \sqrt{d T},
\end{align}
where $\delta_{T} = \max_{1 \leq t \leq T} \operatorname{diam}\left(\left.\mathcal{F}_{t}\right|_{x_{1: t}}\right)$ and $d=\operatorname{dim}_{\mathcal{E}}(\mathcal{F}, \alpha)$.
\end{lemma}


\end{document}